\documentclass{article}

\usepackage{PRIMEarxiv}
\usepackage{natbib}

\usepackage[utf8]{inputenc} 
\usepackage[T1]{fontenc}    
\usepackage{hyperref}       
\usepackage{url}            
\usepackage{booktabs}       
\usepackage{amsfonts}       
\usepackage{nicefrac}       
\usepackage{microtype}      
\usepackage{lipsum}
\usepackage{fancyhdr}       
\usepackage{graphicx}       
\graphicspath{{media/}}     

\usepackage{dcolumn}
\usepackage{bm}
\usepackage{subcaption}
\usepackage{amsthm}
\usepackage{amssymb}
\usepackage{dsfont}
\usepackage{makecell}
\usepackage{multirow}
\usepackage{array}
\usepackage{mathtools}
\usepackage{xcolor}
\usepackage{algorithm}
\usepackage{algpseudocode}
\newtheorem{theorem}{Theorem}
\newtheorem{defs}{Definition}
\newtheorem{cor}{Corollary}

\newtheorem{assum}{Assumption}

\DeclarePairedDelimiter{\abs}{\lvert}{\rvert}
\newcommand{\R}{\mathbb{R}}   
\newcommand{\1}[1]{\mathds{1}\left[#1\right]}
\newcommand{\bs}{\boldsymbol} 
\newcommand{\norm}[1]{\| #1 \|}
\newcolumntype{C}[1]{>{\centering\arraybackslash}p{#1}}
\usepackage{bbold}

\definecolor{darkgreen}{rgb}{0,0.5,0}
\definecolor{corn}{rgb}{0.98, 0.93, 0.36}

\title{Delay-SDE-net: A deep learning approach for time series modelling with memory and uncertainty estimates
}

\author{
  Mari Dahl Eggen \\
  Department of Mathematics \\
  University of Oslo \\
  Oslo, Norway\\
  \texttt{marideg@math.uio.no} \\
   \And
  Alise Danielle Midtfjord \\
  Department of Mathematic \\
  University of Oslo \\
  Oslo, Norway\\
  \texttt{alisedm@math.uio.no} \\
}

\begin{document}
\maketitle

\begin{abstract}
To model time series accurately is important within a wide range of fields. As the world is generally too complex to be modelled exactly, it is often meaningful to assess the probability of a dynamical system to be in a specific state. This paper presents the Delay-SDE-net, a neural network model based on stochastic delay differential equations (SDDEs). The use of SDDEs with multiple delays as modelling framework makes it a suitable model for time series with memory effects, as it includes memory through previous states of the system. The stochastic part of the Delay-SDE-net provides a basis for estimating uncertainty in modelling, and is split into two neural networks to account for aleatoric and epistemic uncertainty. The uncertainty is provided instantly, making the model suitable for applications where time is sparse. We derive the theoretical error of the Delay-SDE-net and analyze the convergence rate numerically. At comparisons with similar models, the Delay-SDE-net has consistently the best performance, both in predicting time series values and uncertainties.
\end{abstract}

\keywords{ Physics-informed neural networks \and Time series \and Uncertainty estimation \and Stochastic delay differential equations \and Barron spaces}

\section{Introduction}
\label{sect:introduction}


Accurate modelling of time series is important within a wide range of fields, such as economics, engineering and natural sciences. As pointed out in \cite{box2008}, some dynamical systems might be modelled nearly exactly in time using physical laws. In this case deterministic models, that are typically solutions to ordinary differential equations (ODEs), do the job. However, the world is generally too complex to be modelled exactly. To deal with this, we can turn to stochastic differential equations (SDEs) to assess the probability of dynamical systems being in a specific state, which provides the possibility of retrieving accurate uncertainty estimates.

Despite the fact that existence and uniqueness results are obtained for many SDEs, see for example \cite{oksendal03} and \cite{protter05}, it is not possible to find an explicit solution of most. A consequence is that dynamical systems are often assumed to follow simple SDEs, as for example by assuming constant or linear drift and diffusion terms. It is even more challenging to consider mathematical models for autocorrelated dynamical systems. As explained in \cite{mohammed01}, the rate of change of many physical systems is not only dependent on the current state of system, but also on past states. For this more intricate case one can use statistical models such as autoregressive moving average (ARMA) models and relatives, see for example \cite{brockwell_davis_book91}. These statistical models assume linear autoregressive interactions in dynamical systems. It has been shown that the ARMA model has a continuous time analogue, also in the multivariate case \citep{marquardt07}. This continuous time analogue is called the continuous time ARMA (CARMA) model, which is a linear system of SDEs with known explicit solution, see for example \cite{brockwell2014recent}.

Even though linear autoregressive models have potential to approximate dynamical systems well, real world systems often have inherent nonlinear temporal dependencies. Many statistical approaches are developed to model such nonlinear systems, see for example \cite{corduas1994nonlinearity}; \cite{kantz_schreiber_2003} and \cite{terasvirta10}. Mathematical models representing nonlinear time dependence are called stochastic delay differential equations (SDDEs)  \citep{mohammed98}, of which the CARMA model is a special case \citep{nielsen20}. Further, in the last decades there has been developed a series of nonlinear time series models built on neural networks, as for example the time delay neural network \citep{waibel1989phoneme}, the hybrid ARIMA and neural network model \citep{zhang2003time}, recurrent neural networks \citep{rumelhart1986learning}, long short-term memory \citep{hochreiter1997long} and encoded time series together with convolutional neural networks \citep{wang2015encoding,gu2018recent,borovykh2017conditional,zheng2014time}. 

More recently, \cite{chen2018neural} contributed to the field of time series modelling by introducing the ODE-net, which is based on continuous dynamics of ordinary differential equations (ODEs), where the derivative is given by a neural network. Even though the ODE-net can be used to model dynamical systems, it does not take previous states into account. That is, the ODE-net has no memory, making it a poor choice in modelling of autoregressive dynamical systems compared to the former mentioned networks. Constraining the neural network with an explicit mathematical structure, such as done in the ODE-net, is referred to as physics-informed neural networks \citep{raissi2019physics}. Being physics-informed may provide several good properties for the neural networks, since the networks are given additional information such as physics-inspired constraints and knowledge. The neural network models may potentially achieve higher interpretability, better robustness on perturbed data, and more stability on out-of-distribution (OOD) data, all on a lower amount of training data \citep{karniadakis2021physics,kashinath2021physics,yang2019adversarial}.

Neural network methods have gained significant attention due to their prediction ability, making them a powerful tool also in time series prediction. They also require minimal data preprocessing compared to some classical models (as for example the ARMA models that require stationary data), speeding up the process of data preparation. However, despite their ability to capture complex patterns and provide high accuracy predictions, neural networks are shown to be unstable. That is, if input data deviate from the training data, such as having small perturbations, structural changes or adversarial inputs, neural networks might give widely wrong and over-confident predictions \citep{antun2020instabilities,papernot2018deep}. Further, basic neural networks do not offer estimates of prediction uncertainty. Uncertainty estimates are particularly important in safety-critical applications such as medicine \citep{lambrou2010reliable,wainberg2018deep}, risk analysis \citep{varshney2016engineering,tambon2022certify} and autonomous systems \citep{pereira2020challenges}, where it is essential to know when to trust a prediction. There are also application areas that involve dynamical systems with large inherent uncertainty which is important to quantify, such as meteorology \citep{scher2018predicting,wang2019deep} and finance \citep{jeon20,ZHANG2020101528}. Because of these requirements a new trend has emerged in machine learning research, where focus on achieving high prediction accuracy has been complemented with proving additional decision support such as explainability \citep{zhang2022explainable,midtfjord2022decision}  and uncertainty estimates \citep{kabir2018neural}.

The most commonly used method for uncertainty estimation is bayesian neural networks \citep{denker1990transforming, mackay1992practical}, where probability distributions are created over model parameters to quantify uncertainty. Inconveniences in using this method include the need of specifying a suitable prior distribution, as well as having to execute many iterations to marginalize the uncertainty over parameters. Other methods used in uncertainty estimation include ensemble methods \citep{kurutach2018model,rajeswaran2016epopt} and test-time augmentation methods \citep{wang2019aleatoric, kim2020learning}, where the former compute the prediction from an ensemble of predictors and use the variance between the predictors as uncertainty estimate, while the latter repeats the prediction from a deterministic network over altered input data. A drawback with all the mentioned methods is the need to repeatedly compute a prediction in order to obtain uncertainty estimates, which can be both computational costly and time consuming. For a review of research on estimating and quantifying uncertainty in neural networks' predictions, see  \cite{abdar2021review} and \cite{gawlikowski2021survey}.

As mentioned in \cite{gawlikowski2021survey}, uncertainty estimation methods falling under the collective term of \textit{single deterministic methods} do not hold
the inconvenience of repeatedly having to compute predictions. Such methods estimate the uncertainty directly, either within the neural network itself \citep{mozejko2018inhibited,nandy2020towards}, or in juxtaposed networks making uncertainty inference about the prediction network \citep{raghu2019direct,oberdiek2018classification}. With single deterministic methods, uncertainty estimates can be made instantly, making them suitable for real-world applications such as text translation, autonomous systems and predictions within finance.

A more recent approach related to providing uncertainty estimates that accompany deterministic predictions is to use neural network models inspired by SDEs, see \cite{tzen2019neural} and \cite{li20}. Additionally, \cite{kong20} contributed to this field by developing the SDE-net. Further innovations on the SDE-net have been developed for more flexible modelling, see for example \cite{hayashi22}; \cite{wang21_1}; \cite{wang21_2} and \cite{yang21}. The SDE-net is similar to the ODE-net in that the derivative is modelled using a neural network, but in the SDE-net an additional neural network is added to model a stochastic term. Even though these neural SDEs can be viewed as a state transformation of dynamical systems, they do not take previous states into account, making them less suitable for modelling of autoregressive dynamical systems. 

Providing good estimates of uncertainty is not necessarily straight forward, as uncertainty can origin from different sources. \textit{Aleatoric uncertainty} refers to the natural randomness inherent in a specific task, while \textit{epistemic uncertainty} arise due to lack of knowledge within a model, which is highly connected to OOD data detection \citep{hullermeier2021aleatoric}. In many applications it is important to distinguish between these two types of uncertainty, as it provides additional information about the trustworthiness of a prediction, and knowledge on how to improve a model \citep{der2009aleatory}. The mentioned SDE-net provides some estimates on both of these uncertainties, but focuses mainly on OOD data detection. In modelling of time series, time-varying variance (heteroskedasticity) is a typical feature. This requires good estimates of aleatoric uncertainty, as periods with larger variance increase uncertainty of models. This is why some SDE-net innovations focus more on handling aleatoric uncertainty \citep{wang2019aleatoric,wang21_2}. 

The above discussions suggest that a time series model would benefit from being able to describe autocorrelated dynamical systems with nonlinear dependencies, as well as giving precise uncertainty estimates. That is, a suitable model could be one taking previous states into account like the ARMA model, describing nonlinearities such as neural networks, and being able to estimate and distinguish between aleatoric and epistemic uncertainty. This work proposes the Delay-SDE-net, an extended version of the SDE-net \citep{kong20} which includes time lagged values of time series. The Delay-SDE-net is based on a stochastic delay differential equation (SDDE) with multiple delays. The model coefficients are given by neural networks, providing a flexible model estimation framework. Additionally, the stochastic part of the Delay-SDE-net is split into two neural networks to represent aleatoric and epistemic uncertainty. The networks are trained according to the single deterministic approach, meaning it can provide
both prediction and uncertainty estimates immediately.

The paper is structured by first providing preliminaries in Section \ref{sect:preliminaries}, which include notation, theory about uncertainty estimates, an introduction of SDDEs, as well as mathematical theory for two-layer neural networks. The Delay-SDE-net is presented in Section \ref{sect:sdde_net_theory}, including its theoretical setup and training algorithm. The theoretical error bound of a two-layer Delay-SDE-net is derived in section \ref{sect:convergence_theory}, and further tested numerically in Section \ref{sect:Analyses_and_application}. This section also compares the Delay-SDE-net to benchmark models, both using simulated data and in a real-world case study. Finally, Section \ref{sect:conclusions} summarizes the results and gives some conclusive remarks.


\section{Preliminaries}
\label{sect:preliminaries}


In this section, theory about model uncertainties will be presented in terms of aleatoric and epistemic uncertainty. Subsequently, a specific class of It{\^o} SDDEs will be presented, as well as the corresponding Milstein scheme. A short introduction of two-layer neural networks and Barron spaces follows. For a more complete introduction of each topic, we refer the reader to given references. We start by introducing notation. 

\subsection{Notation}
\label{subsect:notation}

For convenience of the reader, this section introduces and summarizes notation used in this work. First, we assume that a complete and filtered probability space $(\Omega,\mathcal{F},\{\mathcal{F}_{t\geq 0}\},P)$ is given, under which an $d_{W}$-dimensional standard Brownian motion process $\{W(t)\coloneqq (W^1(t),\ldots ,W^{d_{W}}(t))\}_{t\geq 0}$ is defined. 

Now, define $\mathbb{N}\coloneqq \mathbb{N}_0\backslash \{0\}$, and let $\R^{d}$ be the $d$-dimensional Euclidean space with its standard norm over scalars denoted as $\abs{{}\cdot{}}$, and over vectors as $\norm{{}\cdot{}}_2$. Regular dot product between vectors $u$ and $v$ are denoted $uv$, the corresponding Hadamard product is written as $u\circ v$, and $v^T$ means the transpose of $v$. Further, define the $L_2$-norm $\norm{{}\cdot{}}_{L_2}\coloneqq \\ \norm{{}\cdot{}}_{L_2(\Omega,\R^d)}= (E_{x\sim\mu}[\abs{{}\cdot{}}^2])^{1/2}$ for $x\in \R^d$ with probability measure $\mu\in P(\Omega)$, and the sup-norm $\norm{{}\cdot{}}_C\coloneqq \sup_{-\tau\leq s \leq 0}\abs{{}\cdot{}}$ over all continuous functions (multidimensional paths) $C([-\tau,0];\R^d)$, for a fixed delay $\tau>0$. Continuous functions $C^{i_t,i_x}(T\times \R^{dp},\R^d)$ are having up to $i_t$ and $i_x$ bounded derivatives for time and space respectively. For some bounded domain $D$, we denote by $\mathcal{B}\coloneqq \mathcal{B}_2(D,\R)$ the Barron space of interest in the current work.  The $d$-dimensional stochastic process $X(t)$ driven by $W(t)$ has a corresponding continuous-time Delay-SDE-net model $X^{(m)}(t)$, with $m\in\mathbb{N}$ being the number of neurons in the model coefficients. The discrete-time versions of $X(t)$ and $X^{(m)}(t)$ are denoted by $X^\pi(t)$ and $X^{\pi,(m)}(t)$ respectively, with $\pi$ representing a defined discretization scheme. Elements of such vector processes are denoted with an additional superscript $j=1,\ldots ,d$.

In the empirical analyses, we work with training data sets, $\mathcal{D}_0$, validation and test data sets, $\mathcal{D}$, OOD data sets, $\tilde{\mathcal{D}}$, and modified data sets, $\mathcal{D}^*$. For each time point, we denote by $\bs{x}_{t_k}=[x_{1,t_k},\ldots ,x_{d,t_k}]$ a $d$-dimensional data point. Estimated model coefficients for the Delay-SDE-net are denoted as $f_m$ for drift and $g_m$ for diffusion (or $g_{a,m}$ and $g_{e,m}$ when discussing aleatoric and epistemic diffusion, respectively). Further, estimated model coefficients for other models (or for models in general) are denoted by $\hat{f}$ for drift and $\hat{g}$ for diffusion (or $\hat{g}_{a}$ and $\hat{g}_{e}$ when discussing aleatoric and epistemic diffusion, respectively). Finally, elements of $d$-dimensional vector processes, $X(t)$, and vector data points, $\bs{x}_{t_k}$, are normally distributed with mean $\mu$ and standard deviation $\sigma$ when we write $X(t),\bs{x}_{t_k}\sim N(\mu,\sigma)^d$. 

\subsection{Uncertainty estimates}
\label{subsect:model_uncertainties}

This section presents some important factors related to uncertainty quantification for data-driven models. Especially, we focus on the importance of distinguishing between two different sources of uncertainty, namely aleatoric uncertainty and epistemic uncertainty \citep{der2009aleatory}.

Aleatoric uncertainty refers to uncertainty that origins from natural randomness in a task, which is considered not possible to reduce. In other words, observed data points $y$ from a data generating function $f$ are disturbed by stochastic noise, $y=f(x)+e$, where the uncertainty due to $e \sim N(0,\sigma)$ is the aleatoric uncertainty.

Epistemic uncertainty refers to lack of knowledge within an estimated model $\hat{f}(x)$, and is considered uncertainty possible to reduce by gaining more information. Hence, epistemic uncertainty is the uncertainty due to the difference between the estimated model and the true function, $\hat{f}(x)-f(x)$, and is illustrated in in Figure \ref{fig:epistemic_figure}. The literature often refers to two sources of epistemic uncertainty \citep{sluijterman2021evaluate}. 

\begin{figure}[hbt!]
    \centering
    \includegraphics[width=6cm]{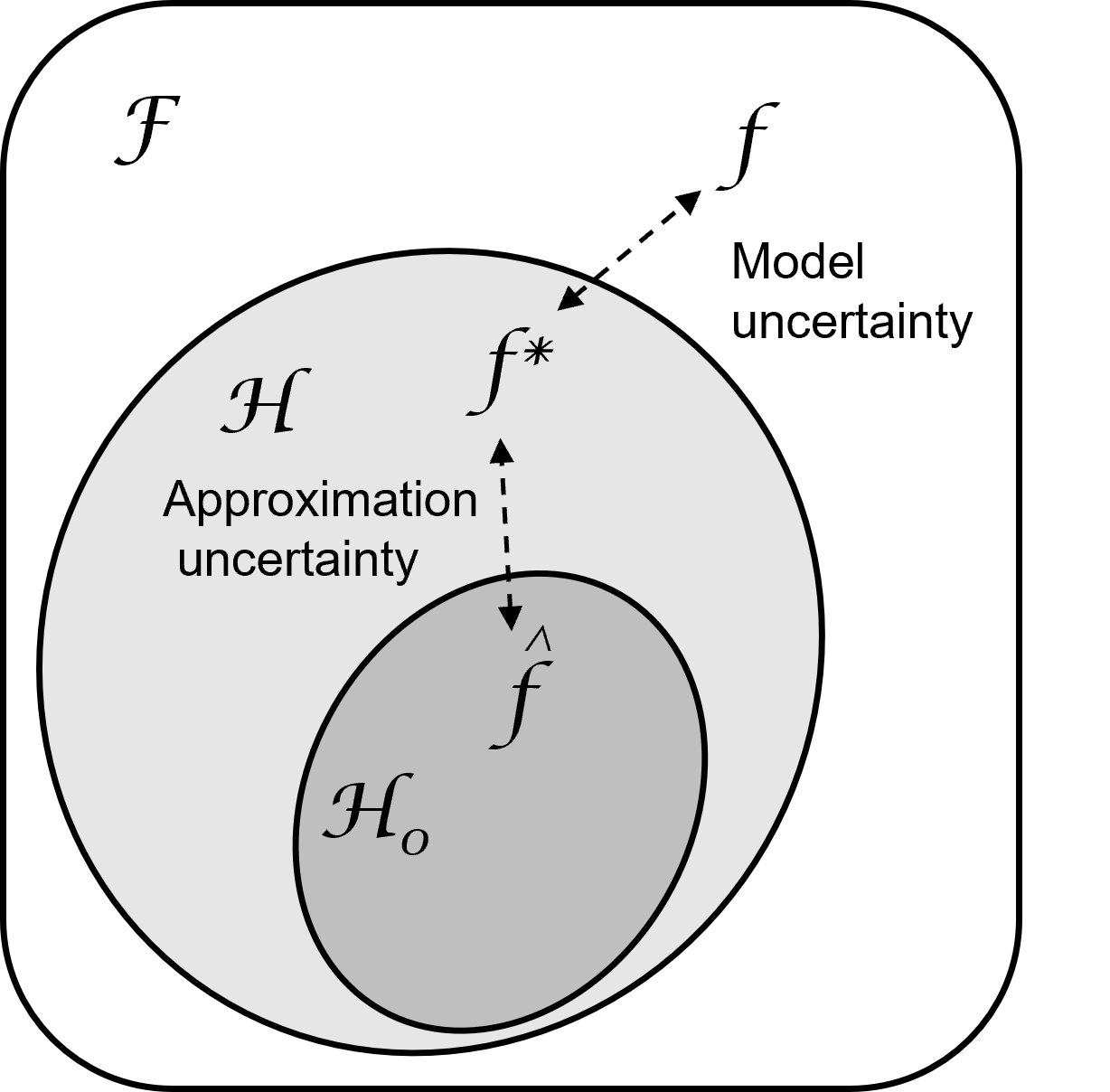}
    \caption{Illustration of the two sources of epistemic uncertainty. $\mathcal{F}$ is the true function space which maps $x$ to $y$. $\mathcal{H}$ is the hypothesis space, namely the class of possible models functions the algorithm can learn. $\mathcal{H}_0$ is contained in $\mathcal{H}$, and is the hypothesis space which the algorithm is trained on.}
    \label{fig:epistemic_figure}
\end{figure}

First, \textit{approximation uncertainty} originates from that the learned model $\hat{f}$ is not trained well enough within the hypothesis space $\mathcal{H}$. This means that we have not found the optimal model $f^*$ within the class of possible models functions. In addition, the hypothesis space the model is trained on, $\mathcal{H}_0$, might not cover the whole hypothesis space $\mathcal{H}$. This can make the model unstable and uncertain outside $\mathcal{H}_0$. The approximation uncertainty can be reduced by increasing the number of training samples, as well as their covered space. 

The second source of epistemic uncertainty is \textit{model uncertainty}, which considers the coverage of the hypothesis space. That is, $\mathcal{H}$ might not cover the true function space $\mathcal{F}$. The consequence of this is that the model might not be able to replicate the true function properly, for example due to lack of model complexity. Possible efforts to reduce the model uncertainty can be to change the model structure, change the learning algorithm, or to provide additional information in terms of new or transformed explanatory variables. Since a model does not know if the true underlying function is contained within its hypothesis space, model uncertainty is often mixed with random noise, that is, aleatoric uncertainty.

To represent uncertainty of model predictions correctly, uncertainty estimates should include both aleatoric and epistemic uncertainty \citep{hora1996aleatory}. However, traditional machine learning methods fail to differentiate between these, and might often give estimates of only one of them \citep{hullermeier2021aleatoric}. For a detailed discussion on the differentiation between aleatoric and epistemic uncertainty, we refer the reader to \cite{der2009aleatory}.

Since neural networks are shown to be able to approximate nearly any function \citep{hornik1989multilayer,lu2017expressive, benth2022neural}, also with as few as only one hidden layer, the (epistemic) model uncertainty could in theory be reduced to zero. The same applies for (epistemic) approximation uncertainty, which in theory could be reduced to zero if enough data points from the whole hypothesis space were provided in training. Despite these facts, it would rarely be possible to train the optimal neural network for most situations in practise, as discussed in \cite{colbrook2022difficult}. Further, to retrieve data from all possible situations is rarely feasible. This means that both aleatoric and epistemic uncertainty estimation, also for neural networks, will continue to be relevant. However, with increasingly better models, and larger amounts of data, aleatoric uncertainty may become the most influencing part.

\subsection{SDDEs with multiple delays and the Milstein scheme}
\label{subsect:sdde_theory}

This section introduces It{\^o} SDDEs with multiple delays as presented in \cite{mohammed04}, as well as the corresponding Milstein scheme.

Some concepts have to be introduced to set up the SDDE framework. First, define a projection $\Pi:C\to \R^{d p}$ associated with given fixed time points $u_1,\ldots ,u_p\in[-\tau,0]$ as
\begin{align}
    \label{eq:the_projection}
    \Pi(\eta)\coloneqq(\eta(u_1),\ldots ,\eta(u_p))\in\R^{d p},
\end{align}
with $\eta: \Omega \to C$ being a $\mathcal{F}_0$-measurable initial process. Further, for any continuous $d$-dimensional stochastic process $X(t):[-\tau,T]\times\Omega\to\R^d$, where $\tau, T\in\R^{+}$, we define its segment process by 
\begin{align}
    \label{eq:the_segment_process}
    X_t(u)\coloneqq X(t+u),\quad u\in [-\tau,0],\quad t\in [0,T],
\end{align}
with $X_t$ taking values in $C$.

Given two projections $\Pi_1$ and $\Pi_2$ with associated time points $u_{1,1},\ldots ,u_{1,p_1}\in [-\tau,0]$ and $u_{2,1},\ldots ,u_{2,p_2}\in [-\tau,0]$ respectively, we define the following class of SDDEs. Let the stochastic (vector) process $X(t)$ be given by the (unique strong) solution of 
\begin{align}
    \label{eq:sdde_partitioned}
    X(t) = \begin{cases}
    & \eta(0) + \int_0^t f(s,\Pi_1(X_s))ds + \int_0^t g(s,\Pi_2(X_s))dW(s) ,\quad t\geq0,\\
    & \eta(t),\quad -\tau\leq t < 0,
    \end{cases}
\end{align}
where $W(t)$ is a Brownian motion process \citep[see for example][]{oksendal03} taking values in $\R^{d_{W}}$, and the coefficients $f:[0,T]\times \R^{d p_1}\to \R^d$, $g:[0,T]\times\R^{d p_2}\to \R^{d\times d_W}$ are given by tame functions \citep[see][]{mohammed04}, and satisfy the Lipschitz and boundedness conditions in \cite{mohammed04} (Eq.\,(1.4) and (1.5) respectively). Then, for each $q\geq 1$, there exists a constant $C_q=C_q(q,C_L,T)>0$, with $C_L$ being a Lipschitz constant, such that 
\begin{align*}
    E\norm{X_t}^{2q}_C \leq C_q(1+E\norm{\eta}^{2q}_C),
\end{align*}
for all $\eta\in C$, $t\in [0,T]$, where $\norm{{}\cdot{}}_C\coloneqq \sup_{-\tau\leq s \leq 0}\abs{{}\cdot{}}$ (see Section\,\ref{subsect:notation}).

Define an equidistant discretization of the time interval $[-\tau,T]$ as
\begin{align}
    \label{discretization_partition}
    \mathcal{T} = \{-\tau = t_{-L}<t_{-L+1}<\cdots < 0 = t_0 < t_1 <  \cdots < t_K = T\}, 
\end{align}
for integers $L\geq 0$, $K\geq 1$, and $\pi\coloneqq \Delta t = t_{k+1}-t_k$ for $k=\{-L,\ldots ,K-1\}$. We assume that $\mathcal{T}$ is defined to include the set of delays, $\{u_{1,1},\ldots ,u_{1,p_1},u_{2,1},\ldots ,u_{2,p_2}\}\in \mathcal{T}$. Further, denote by $f^j(t,x^{(1)})$ and $g^{ji}(t,x^{(2)})$ element $j$ and $(j,i)$ of $f\in\R^d$ and $g\in\R^{d\times d_W}$ respectively, where $x^{(1)}\in\R^{d p_1}$ and $x^{(2)}\in\R^{d p_2}$. For the case $f\in C^{1,2}(T\times \R^{d p_1},\R^d)$, $g\in C^{1,2}(T\times \R^{d p_2},\R^{d\times d_W})$ and $W(s)=W(0)=0$, $s\leq 0$, we define the Milstein scheme of the SDDE in Eq.\,\eqref{eq:sdde_partitioned} as
\begin{align}
\begin{split}
    \label{eq:multi_dim_milstein_scheme}
    X^{j,\pi}(t_{k+1}) = & X^{j,\pi}(t_k) + f^{j}(t_k,\Pi_1(X_{t_k}^{\pi}))\Delta t +g^{ji}(t_k,\Pi_2(X_{t_k}^{\pi}))\Delta W^i(t)\\ & + F^j\left(g^{j_2i_2},\frac{\partial g^{ji}}{\partial x_{j_2i_2}}\right),
    \end{split}
\end{align}
where $X^{j,\pi}(t)$ is the discretized version of the $j$-th component of $X(t)\in\R^d$, and the superscript $\pi$ indicates that the time discretization in Eq.\,\eqref{discretization_partition} is applied. Finally, we have $\Delta W^i(t)\coloneqq W^i(t_{k+1})-W^i(t_k)\in \R^{d_W}$. See \cite{mohammed04} for exact expression of $F(\cdot,\cdot)$, and note that $F^j = 0$ whenever $\partial g^{ji}/\partial x_{j_2i_2}=0$. 

\subsection{Two-layer neural networks and Barron spaces}
\label{subsect:nn_barron_space_theory}

Existence and error bounds of functions given by two-layer neural networks are discussed in the spirit of Barron spaces. As stated in \cite{weinan21}, Barron spaces are defined as the set of continuous functions that can be represented as the continuous-time version of a two-layer neural network, having finite Barron norm. That is, a continuous function approximated by a two-layer neural network with approximation error bounded by the Barron norm, is in Barron space. A more formal definition follows.

As presented in \cite{weinan2019_pop_risk}, we define a target function $f^{*} = E[y\mid x]$, for $x\in\R^d$, $y\in\R$, that can be learned from i.i.d. data samples $\{\bs{x}_i,y_i\}_{i=0}^{n}$, $n\in\mathbb{N}$, drawn from an underlying distribution $p_{x,y}$. The target function $f^*$ is illustrated in a function space setting in Figure\,\ref{fig:epistemic_figure}. We refer to data from $p_{x,y}$ as in-distribution (ID) data. In contrary, OOD data refers to data drawn from another probability distribution than the ID data, denoted $\tilde{p}_{x,y}$. The following introduction of two-layer neural networks and Barron spaces holds for both data types. An explanation of how to use OOD data in training of neural networks is given in Section\,\ref{subsubsect:train_g_e}.

A two-layer neural network can be represented as 
\begin{align}
    \label{eq:two-layer_nn}
    f_m(x;\theta) = \sum_{i=1}^{m} a_i \sigma(w_i^T x + b_i),
\end{align}
with $\theta = \{a_i,w_i,b_i\}_{i=1}^{m}$, $(a_i,w_i,b_i)\in \R\times \R^{d} \times \R\coloneqq \hat{\Omega}$, denoting the model parameter space, $\sigma:\R\to\R$ being an activation function, and $m$ the number of hidden neurons. Note that we use $w_i\in \R^{d+1}$ in applications in this work to include time as a dimension. The infinite width limit of the hidden layer represented in Eq.\,\eqref{eq:two-layer_nn} is taken as a function $f_\rho(x):D\to \R$, for a bounded domain $D\subset \R^{d}$, given by
\begin{align}
    \label{eq:two-layer_nn_continuous}
    f_\rho(x) = \int_{\hat{\Omega}}a\sigma(w^T x + b)\rho(da\otimes dw \otimes db) = E_{(a,w ,b)\sim \rho}\left[a\sigma (w^T x + b)\right].
\end{align}
Here, the probability measure $\rho \in \hat{P}(\hat{\Omega})$ is defined on a measure space $(\hat{\Omega},\hat{\mathcal{F}},\hat{P})$, with $\hat{\mathcal{F}}$ being the Borel $\sigma$-algebra on $\hat{\Omega}$.

In \cite{weinan21}, the Barron norm for a general activation function $\sigma$ is defined as
\begin{align*}
    \norm{f}_{\mathcal{B}_q} \coloneqq \inf_{\rho} \left(E_\rho\left[\abs{a}^q(\norm{w}_1+\abs{b}+1)^q\right]\right)^{1/q},
\end{align*}
with infimum over $\rho$ such that $f =f_\rho$ holds, where $q\in[1,\infty]$. Further, Barron space is defined as the set of continuous functions of the form as in Eq.\,\eqref{eq:two-layer_nn_continuous} satisfying $\norm{f}_{\mathcal{B}_q}<\infty$. For a more extensive introduction of Barron spaces, see for example \cite{weinan2019_pop_risk}; \cite{weinan21}; \cite{weinan21_barron_space_appendix}; \cite{weinan22_pde} and \cite{weinan22_barron_pointwise_properties}.

The theoretical model error bound for the Delay-SDE-net (introduced in Section\,\ref{sect:sdde_net_theory}) is derived in Section\,\ref{sect:convergence_theory}. This derived result holds for Barron functions only, see Theorem\,\ref{thm:continuous_Delay-SDE-net_error} and corresponding proof.  This is because functions in Barron spaces have the useful property that their two-layer neural network error bounds are given (inversely) as a function of the number of neurons, $m$, in the neural network (see Eq.\,\eqref{eq:two-layer_nn}). Note that we use the Barron space with $q=2$ in this work, and write $\mathcal{B}\coloneqq\mathcal{B}_{2}$ for convenience. Before stating a theorem giving an error bound for Barron functions we need the following definition.

\begin{defs}
\label{def:the_gamma-function}
Given an activation function $\sigma$, define the $\gamma(\cdot)$-function as 
\begin{align*}
    \gamma(\sigma) = \gamma_0(\sigma) + \inf_{x\in\R}u(x),
\end{align*}
where
\begin{align*}
    \gamma_0(\sigma) &= \int_\R \abs{\sigma''(x)}(\abs{x} + 1)dx\quad\text{and}\quad
    u(x) = \abs{\sigma(x)} + (\abs{x}+2)\abs{\sigma'(x)}.
\end{align*}
\end{defs}

\begin{theorem}[\cite{ma20}, Theorem\,4]
\label{thm:convergence_barron_functions}
For any $f\in\mathcal{B}$ and $m\in\mathbb{N}$ there exists a two-layer neural network $f_m(\cdot;\theta)$ with finite width $m$ and activation function $\sigma$ such that
\begin{align*}
    \norm{f_m(x;\theta)-f(x)}_{L_2} \leq \frac{\sqrt{3C_\sigma}\norm{f}_\mathcal{B}}{\sqrt{m}},
\end{align*}
where $C_\sigma = \left(\gamma(\sigma) + \min\{\abs{\sigma'(+\infty)},\abs{\sigma'(-\infty)}\} + \abs{\sigma(0)}\right)^2$.
\end{theorem}

Note that Theorem\,\ref{thm:convergence_barron_functions} omits one of the proved results stated in Theorem\,4 of \cite{ma20}. Also consult the proof of the theorem in \cite{ma20}, App.\,C.2.


\section{The Delay-SDE-net}
\label{sect:sdde_net_theory}


The Delay-SDE-net is developed to capture time-lagged dependencies between inputs to a neural network model, as well as to describe potential errors in modelling using the concepts of aleatoric and epistemic uncertainty.

The original SDE-net \citep{kong20} uses one previous discretization step to estimate the next. This method neglects potential memory in data, meaning that information may be lost. The SDDE with multiple delays (presented in Section\,\ref{subsect:sdde_theory})
gives a mathematical framework to generalize the original SDE-net to what we call the Delay-SDE net. That is, the Delay-SDE-net allows for estimation based on $p$ previous discretization steps, adding information about trends or rapid changes, making the model more capable of understanding the present situation.  The way Delay-SDE-net takes advantage of information from lagged time values resembles the family of (linear) ARMA models. However, Delay-SDE-net
has the added advantage of modelling drift and diffusion using neural networks, where nonlinear relationships can be represented, meaning the hypothesis space $\mathcal{H}$ of the model is extended.

 Additionally, as we will see in Section \ref{subsect:training_algorithm}, the Delay-SDE-net models the diffusion coefficient using two neural networks, one for the aleatoric uncertainty and one for the epistemic uncertainty, allowing it to distinguish between the two uncertainty types. The way these neural networks are trained makes the Delay-SDE-net part of the group of single deterministic methods for uncertainty quantification, meaning that uncertainty estimates are given instantly without being delayed by a series of simulations.

\subsection{Theoretical setup of the Delay-SDE-net}
\label{subsect:theoretical_setup}

As defined in \cite{kong20}, the original SDE-net constraints a neural network model using an It{\^o} diffusion of the form 
\begin{align*}
    d{X}(t) = f(t,X(t))dt + g(t,X(t))dB(t),
\end{align*}
where $B(t) = \{B(t)\}_{t\geq 0}$ is a one-dimensional Brownian motion process \citep[see for example][]{oksendal03}. In this section, we define a generalized version of the SDE-net. The extended modelling framework is referred to as the Delay-SDE-net, as the neural network model is constrained by an SDDE (see Sect.\,\ref{subsect:sdde_theory}). The Delay-SDE-net generalizes the original SDE-net to consider autoregressive dependencies between one or multiple lags in a multidimensional dynamical system.

When defining the Delay-SDE-net, we assume that given data are generated by an SDDE as given in Eq.\,\eqref{eq:sdde_partitioned}, To make the modelling framework reasonably simple, we require some constraints on the SDDE. That is, we assume that given data are generated by a stochastic process, $X(t)$, given by the solution of the SDDE
\begin{align}
    \label{eq:data_generating_sdde}
    X(t) = \begin{cases}
    & \eta(0) + \int_0^t f(s,\Pi(X_s))ds + \int_0^t g(s,\Pi(\eta(0)))dW(s) ,\quad t\geq 0,\\
    & \eta(t),\quad -\tau\leq t < 0,
    \end{cases}
\end{align}
where $\Pi:C\to \R^{d p}$ is as defined in Eq.\,\eqref{eq:the_projection}. Notice that the diffusion is state-independent, as it is only dependent on the initial path $\eta:\Omega \to C$. This assumption is made to obtain a simplified discretization scheme of the model. That is, we obtain $F^i=0$, for all $i=1,\ldots , d$, in Eq.\,\eqref{eq:multi_dim_milstein_scheme}, meaning that the Milstein scheme is equivalent to to Euler-Maruyama scheme \citep{KloedenPeterE1992Nsos}.

In the following, we define the model that will be referred to as the Delay-SDE-net. Assuming that given data are generated by an SDDE on the form as given in Eq.\,\eqref{eq:data_generating_sdde}, it is reasonable to construct a model
\begin{align}
    \label{eq:continuous_nn_sdde}
    X^{(m)}(t) = \begin{cases}
    & \eta(0) + \int_0^t f_m(s,\Pi(X_s^{(m)}))ds + \int_0^t g_m(s,\Pi(\eta(0)))dW(s) ,\quad t\geq 0,\\
    & \eta(t),\quad -\tau\leq t < 0,
    \end{cases}
\end{align}
with initial path $\eta:\Omega \to C$ identical to $\eta$ in Eq.\,\eqref{eq:data_generating_sdde}, for the evolution of the dynamical system that the data represent. In Eq.\,\eqref{eq:continuous_nn_sdde}, $f_m$ is given by an $m$-neuron neural network. Further, inspired by the discussion in Sect.\,\ref{subsect:model_uncertainties}, we define 
\begin{align}
    \label{eq:stochastic_part_delay_sde_net}
    g_m \coloneqq g_{a,m} + g_{e,m},
\end{align}
where $g_{a,m}$ and $g_{e,m}$ are given by two independent $m$-neuron neural networks, respectively representing aleatoric and epistemic uncertainty.

Note that, to ease notation, we have assumed that the neural networks $f_m$, $g_{a,m}$ and $g_{e,m}$ are trained using the same number of neurons, $m$. This assumption can be omitted without problems in applications. Further, to analyze the theoretical error bound of the Delay-SDE-net, we will work with $f_m$, $g_{a,m}$ and $g_{e,m}$ as two-layer neural networks in Section\,\ref{sect:convergence_theory}. In the empirical analyses and applications of Section\,\ref{sect:Analyses_and_application}, as well as for applications in general, the neural networks are not restricted to only two layers. Note that we use two-layer neural networks in all analyses of Section\,\ref{sect:Analyses_and_application} of this work for simplicity. Finally, as introduced in Section\,\ref{subsect:nn_barron_space_theory}, it is important to note that the model in Eq.\,\eqref{eq:continuous_nn_sdde} is a continuous-time model, however, $f_m$ and $f_\rho$ are different functions, as $f_\rho$ is the $m\to\infty$ limit of $f_m$.

The continuous-time version of the Delay-SDE-net in Eq.\,\eqref{eq:continuous_nn_sdde} is estimated using its discrete-time approximation 
\begin{align}
    \label{eq:discretized_nn_sdde}
    \Delta X^{j,\pi,(m)}(t_{k}) =   f_m^{j}(t_{k-1},\Pi(X_{t_{k-1}}^{\pi,(m)}))\Delta t +  g_m^{j}(t_{k-1},\Pi(\eta^{\pi}(0)))\Delta W^j(t) ,\quad t\geq 0,
\end{align}
for $j=1,\ldots ,d$, where $g_m^{j}\coloneqq g_{a,m}^{j}+g_{e,m}^{j}$. Note that this discrete-time approximation is based on the Milstein scheme in Eq.\,\eqref{eq:multi_dim_milstein_scheme}, and that the split of the diffusion term was initially defined in Eq.\eqref{eq:stochastic_part_delay_sde_net}. In this case, the observed process $\eta^\pi$ is, for example, taken as a piecewise linear approximation of $\eta$ using the defined partition in Eq.\,\eqref{discretization_partition}. Also note that, for simplicity, we assume that there is one independent driving process for each element of the $d$-dimensional process $X(t)$, meaning that $W(t)$ takes values in $\R^d$. As a result, $g$ and $g_m$ can be taken as vector functions $g\in\R^d$ and $g_m\in\R^d$ respectively (when applying the Hadamard product in vector notation). Then, as long as each $g^j$ (and $g^j_m$), $j=1,\ldots ,d$, satisfies required conditions for existence and uniqueness, and required regularity conditions for discretization, the component-wise stochastic parts of the SDDEs in Eq.\,\eqref{eq:data_generating_sdde} (and Eq.\,\eqref{eq:continuous_nn_sdde}) are $d$ independent $1$-dimensional processes. 

The Delay-SDE-net modelling framework defined in this section can be trained for applications using the methodology in the following section.

\subsection{The training algorithm}
\label{subsect:training_algorithm}

This section presents a methodology for training the Delay-SDE-net. We will see that the deterministic and stochastic parts of Eq.\,\eqref{eq:continuous_nn_sdde} are trained as separate neural networks, and that several new approaches are used compared to related literature.

Sections\,\ref{subsubsect:train_f}, \ref{subsubsect:train_g_a} and \ref{subsubsect:train_g_e} explain how separate neural networks are trained for the deterministic part, $f_m$, and stochastic parts, $g_{a,m}$ and $g_{e,m}$, of the Delay-SDE-net. Note that the the stochastic part is separated into two neural networks as shown in Eq.\,\eqref{eq:stochastic_part_delay_sde_net}, making the Delay-SDE-net a composition of three neural networks. To train the deterministic and stochastic parts separately differs from the approach used for the SDE-net \citep{kong20} as well as its other variants \citep{hayashi22, wang21_1, wang21_2, yang21}. The methodology of separate neural networks makes the training process for $f_m$ more stable since the training of $g_{a,m}$ and $g_{e,m}$ do not affect the deterministic part \citep{gawlikowski2021survey}. In addition, the training of $g_{a,m}$ also becomes more stable, as it is trained on constant residuals from a final optimized $f_m $ (see Section\,\ref{subsubsect:train_g_a}). The latter is shown in \cite{raghu2019direct}, which argues that the uncertainty estimation is biased by the prediction task when both are estimated using a single neural network. Training of two separate neural networks for prediction and uncertainty estimations was also done by \cite{hsu2020generalized}; \cite{oberdiek2018classification}; \cite{lee2020gradients} and \cite{ramalho2020density}, however only for the classification task.

The presented training methodology for $f_m$, $g_{a,m}$ and $g_{e,m}$ in the respective Sections\,\ref{subsubsect:train_f}-\ref{subsubsect:train_g_e} holds for a Delay-SDE-net giving one-day-ahead predictions (predicting one step forward in time). In Section\,\ref{subsubsect:prediction_N_steps_forward}, we present a training methodology for a Delay-SDE-net that predicts $N$ steps forward in time. The training algorithm is given in Algorithm \ref{alg:non-image}, and a corresponding illustration is given in Figure \ref{fig:training}. The code for the algorithm can be found on
 \url{https://github.com/alimid/Delay-SDE-net}.

\begin{figure}[hbt!]
    \centering
    \includegraphics[width=10cm]{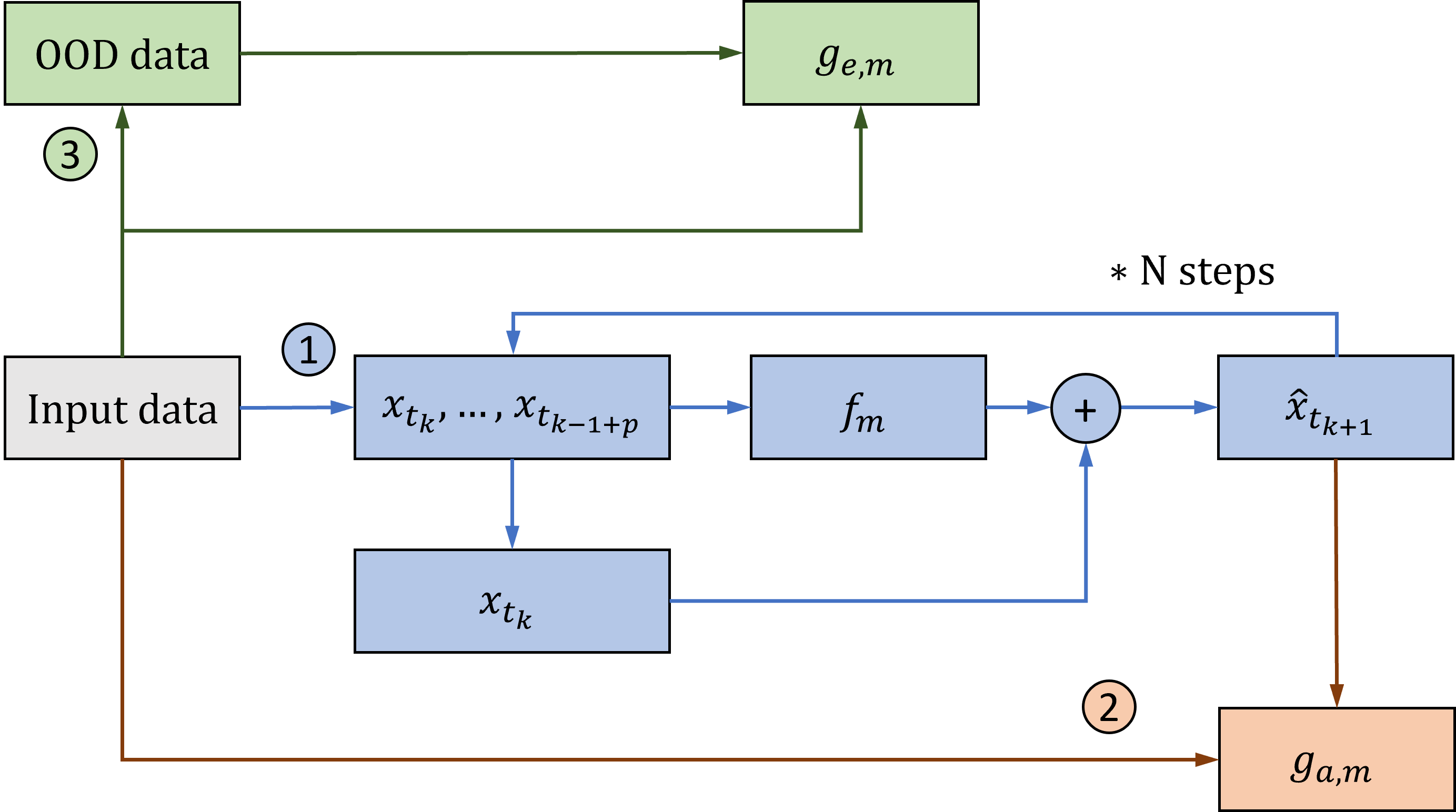}
    \caption{Illustration of the training algorithm for the Delay-SDE-net. The drift term $f_m$ is trained first. Secondly,  the residuals from the drift term is used to train the aleatoric diffusion term $g_{a,m}$. Third, the epistemic diffusion term $g_{e,m}$ is trained separately from the other two terms only on the input data and created OOD data.}
    \label{fig:training}
\end{figure}

\subsubsection{The drift net}
\label{subsubsect:train_f}

The drift net $f_m$, is trained deterministically with time and previous states as input. This would be comparable to a delayed ODE-net, an extended version of the ODE-net introduced in \cite{chen2018neural}. The recent publication \cite{yuan22} introduces a delayed ODE-net where historical information adds to the model as an initial state, an approach different from the method introduced in this work where new historical information is added continuously.

First, note that the discretized Delay-SDE-net in Eq.\,\eqref{eq:discretized_nn_sdde} is used for training. To train $f_m$, we focus on predicting the expected value and ignoring the stochastic term. That is, the optimal drift net is found by minimizing the expected prediction error
\begin{align*}
\min_{\theta_{f_{m}}}E_{\mathcal{D}_0}\left[\norm{\hat{\bs{x}}_{t_{k}}-\bs{x}_{t_{k}}}_2\right],
\end{align*}
where $\mathcal{D}_0\coloneqq\{\bs{x}_{t_{k}}\in\R^d:\quad t_{k}\in [0,T]\}$ is the training data set, $\theta_{f_{m}}$ represents the drift net parameters, and
\begin{align}
\label{eq:training_values_of_drift_net}
\hat{\bs{x}}_{t_{k}}=\bs{x}_{t_{k-1}}+f_m(t_{k-1},\bs{x}_{t_{k-p}},\ldots, \bs{x}_{t_{k-1}})\Delta t.
\end{align}
Inspired by Eq.\eqref{eq:discretized_nn_sdde} we use predictions $\hat{\bs{x}}_{t_{k}}^{\text{optimal}}$ from the optimized drift net to train the aleatoric diffusion net $g_{a,m}$. That is, we use the relation
\begin{align}
\label{eq:optimized_values_of_drift_net}
    X^{\pi,(m)}(t_k) - g_{m}(t_{k-1},\cdot)\circ\Delta W(t) \coloneqq  \hat{\bs{x}}_{t_{k}}^{\text{optimal}} = \bs{x}_{t_{k-1}}+f_m(t_{k-1},\bs{x}_{t_{k-p}},\ldots, \bs{x}_{t_{k-1}})\Delta t,
\end{align}
where $\circ$ is the Hadamard product and we assume $\Delta W\sim N(0,1)^d$. See the following section.

\subsubsection{The aleatoric diffusion net}
\label{subsubsect:train_g_a}

Training of the Delay-SDE-net diffusion term happens when iterations for the drift term are finished. As already mentioned, we train the aleatoric diffusion net, $g_{a,m}$, and the epistemic diffusion net, $g_{e,m}$, separately. The aleatoric diffusion net should capture patterns in data representing the standard deviation of the stochastic noise, corresponding to a measure of the drift net aleatoric uncertainty.

Since the Delay-SDE-net is driven by a (standard) Brownian motion process (see Eq.\,\eqref{eq:optimized_values_of_drift_net}), the stochastic term corresponding to the aleatoric uncertainty is distributed as
\begin{align*}
    \int_0^T g_{a,m}(t,\cdot)dW(t)\sim N\left(0,\int_0^T g_{a,m}^2(t,\cdot)dt\right),
\end{align*}
see for example \cite{BenthFredEspen2004OTwS}, meaning that the approximated variance one step forward in time is given by
\begin{align}
    \label{eq:theoretical_variance_approximation}
    \int_{t_{k}}^{t_{k+1}} g_{a,m}^2(t,\cdot)dt \simeq  g_{a,m}^2(t_k,\cdot)\Delta t ,
\end{align}
where $g_{a,m}(t_k,\cdot)$ is the discretized aleatoric diffusion net in Eq.\,\eqref{eq:discretized_nn_sdde}. We train $g_{a,m}$ on the residuals $\bs{e}_{t_k} = \hat{\bs{x}}_{t_k}^{\text{optimal}} - \bs{x}_{t_k}$ of the predicted values in Eq.\,\eqref{eq:optimized_values_of_drift_net}, and optimize as
\begin{align*}
\min_{\theta_{g_{a,m}}}E_{\mathcal{D}_0}\left[\norm{g_{a,m}^2(t_k,\cdot)\Delta t-\bs{e}_{t_k}^2}_2\right],
\end{align*}
where $\theta_{g_{a,m}}$ represents the aleatoric diffusion net parameters. Note that $\Delta t$ is explicitly given and remains constant during training. Due to this methodology, the optimized aleatoric diffusion net $g_{a,m}$ corresponds to the predicted standard deviation of the model in Eq.\,\eqref{eq:optimized_values_of_drift_net}.

A model allowing time-dependent variance is useful when dealing with heteroscedastic data sets, that is, data sets with nonconstant variance. Typical examples include seasonal behaviour in modelling of weather \citep{eggen2021}, as well as intraday fluctuations in modelling of stock returns \citep{andersen19}. Note that prediction residuals include model error as well as inherent randomness, meaning that $g_{a,m}$ is optimized to capture both types of uncertainty. As discussed in Section\,\ref{subsect:model_uncertainties}, $g_{a,m}$ can not differentiate which parts of the true function that are in the hypothesis space $\mathcal{H}$ and not (see Figure\,\ref{fig:epistemic_figure}). Therefore, $g_{a,m}$ must be considered as modelling the aleatoric noise given the model $g_{a,m}=E[g_a|f_m]$. Even though this is rarely mentioned in literature, model uncertainty is part of all models with diffusion optimzed on model residuals, such as for example the vector autoregressive (VAR) model in \cite{eggen2021}. 

\subsubsection{The epistemic diffusion net}
\label{subsubsect:train_g_e}

The final stage in training the Delay-SDE-net is to train the neural network $g_{e,m}$, corresponding to epistemic uncertainty. In Section\,\ref{subsect:model_uncertainties}, epistemic uncertainty is defined to be composed of two types of uncertainty; approximation uncertainty and model uncertainty. As explained in Section\,\ref{subsubsect:train_g_a}, due to the training methodology of the aleatoric diffusion net $g_{a,m}$, model uncertainty is already part of the model. That is, the trained epistemic diffusion net should reflect approximation uncertainty only.

Approximation error origins from limited availability of training data from the whole hypothesis space, meaning the real distribution is not covered well enough. Based on Section\,\ref{subsect:model_uncertainties}, the epistemic diffusion net should ideally be able to differentiate between data from the hypothesis space it is trained on, $\mathcal{H}_0$, and data it has not seen before, $\not\subset\mathcal{H}_0$ (see Figure\,\ref{fig:epistemic_figure}). We train $g_{e,m}$ with a similar approach as used for the SDE-net \citep{kong20}. That is, we use a classifier distinguishing between ID data and OOD data, and optimize on
\begin{align}
    \label{eq:loss_epistemic_diffusion_net_prob}
\min_{\theta_{g_{e,m}}}E_{\mathcal{D}_0}\left[g_{e,m}^{\text{prob}}(t_k,\cdot)\right]+\max_{\theta_{g_{e,m}}}E_{\tilde{\mathcal{D}}_0}\left[g_{e,m}^{\text{prob}}(t_k,\cdot)\right],
\end{align}
where $\tilde{\mathcal{D}}_0$ represents the OOD data set. Note that the function optimized in Eq.\,\eqref{eq:loss_epistemic_diffusion_net_prob} is $g_{e,m}^{\text{prob}}\in [0,1]^{d}$, and not $g_{e,m}\in\R^{d}$ in Eq.\,\eqref{eq:discretized_nn_sdde}. That is because, when using a probabilistic classification task between the two classes ID-data and OOD-data, the trained network $g_{e,m}^{\text{prob}}$ describes the probability of an observation to be OOD. The epistemic standard deviation, $g_{e,m}$, in Eq.\,\eqref{eq:stochastic_part_delay_sde_net} is finally approximated as $g_{e,m} = \sigma_e g_{e,m}^{\text{prob}}$, $\sigma_e\in\R$, where the appropriate constant $\sigma_e$, representing the maximum value of the epistemic uncertainty, is either set by using knowledge about the created OOD data, or computed as a tuning parameter. In both cases the magnitude of $\sigma_e$ should control the magnitude of the epistemic uncertainty, and balance it in conjunction with the magnitude of the aleatoric uncertainty. When $\sigma_e$ is a tuning parameter we minimize 
\begin{align*}
\min_{\sigma_e}E_{\mathcal{D}}\left[\norm{(\sigma_e g_{e,m}^{\text{prob}}(t_k,\cdot) + g_{a,m}(t_k,\cdot)\sqrt{\Delta t})^2 - \bs{e}_{t_k}^2}_2\right],
\end{align*}
where $g_{a,m}(t_k,\cdot)\sqrt{\Delta t}$ represents the trained aleatoric standard deviation and $\mathcal{D}$ is a validation data set.

If one has access to real OOD data, this could be used to train $g_{e,m}^{\text{prob}}$.  In many applications, however, an OOD data set is rarely available, meaning the OOD data should be created artificially.
The approach used to create OOD data in training of the SDE-net is to draw random data points from the training data set and add Gaussian noise.  That is, the SDE-net is trained using an OOD data set $\tilde{\mathcal{D}}_0$ with elements $\tilde{\bs{x}}_{t_k}=\bs{x}_{t_k}+\bs{\epsilon}_{t_k}$, where $\bs{x}_{t_k}\in\mathcal{D}_0$ and $\bs{\epsilon}_{t_k}\sim N(0,1)^d$. As seen in Figure\,\ref{fig:ood_sde}, this method does not guarantee creation of OOD data as they might overlap with the training data. Therefore, the Delay-SDE-net is instead trained on OOD data created according to the soft Brownian offset method presented in \cite{moller2021out}. This method draws random data points from the training data $\bs{x}_{t_k}\sim \mathcal{D}_0$ and iteratively adds Gaussian noise, $\bs{e}=[e_1,\ldots ,e_d]\sim N(\mu,\sigma)^d$, until a prespecified minimum distance $d^-$ from all data points in the training data set is obtained. Simplified, the method does: \\

\begin{algorithmic}
    \While{$d^*<d^-$}
    \State $\tilde{\bs{x}}_{t_k}=\tilde{\bs{x}}_{t_k}+d^+\bs{e},$
\EndWhile \\
\end{algorithmic}

\noindent where $d^*$ is the distance from $\tilde{\bs{x}}_{t_k}$ to the closest data point in $\mathcal{D}_0$, and $d^+$ is a prespecified length of the offset for the added noise. Data from the soft Brownian offset method, $\tilde{\mathcal{D}}_0=\{\tilde{\bs{x}}_{t_1},\ldots,\tilde{\bs{x}}_{t_{K}}\}$, is guaranteed to be out of distribution from the given training data, as illustrated in Figure\,\ref{fig:ood_delaysde}.

Note that the method used to create OOD data should be tailored to the specific use case, to make sure that potential rare outcomes in the specific application are covered. In the soft Brownian offset method one can for example adjust the minimum distance, $d^{*}$, between ID and OOD data (which has to increase exponentially with the number of dimensions in the data set due to curse of dimensionality), or adjust the spread of the OOD data by changing the offset, $d^{+}$. Further, one can specify what properties in ID data that the OOD data should diverge from. In this work we chose to generate OOD data covering two specific properties. First, we generate OOD data reflecting abnormal local variability in the time series. That is, the soft Brownian offset method is sensitive to differences between lags, because the OOD data set is constructed according to $[\tilde{\bs{x}}_{{t}_k},\ldots ,\tilde{\bs{x}}_{t_{k+1-p}}]^T=[\bs{x}_{{t}_k}+d^+ \bs{e}_1,\ldots ,\bs{x}_{t_{k+1-p}}+d^+ \bs{e}_p]^T$, where $\bs{e}_i\sim N(\mu,\sigma)^d$ for $i=1,\ldots p$. Secondly, we generate OOD data reflecting abnormal magnitude of data points at a given time. That is, the OOD data set is constructed according to $\tilde{\bs{x}}_{{t}_k}=\bs{x}_{{t}_k}+d^+\bs{e}$.

\begin{figure}[hbt!]
\centering
\begin{subfigure}{.49\textwidth}
    \centering
    \includegraphics[width=.95\linewidth]{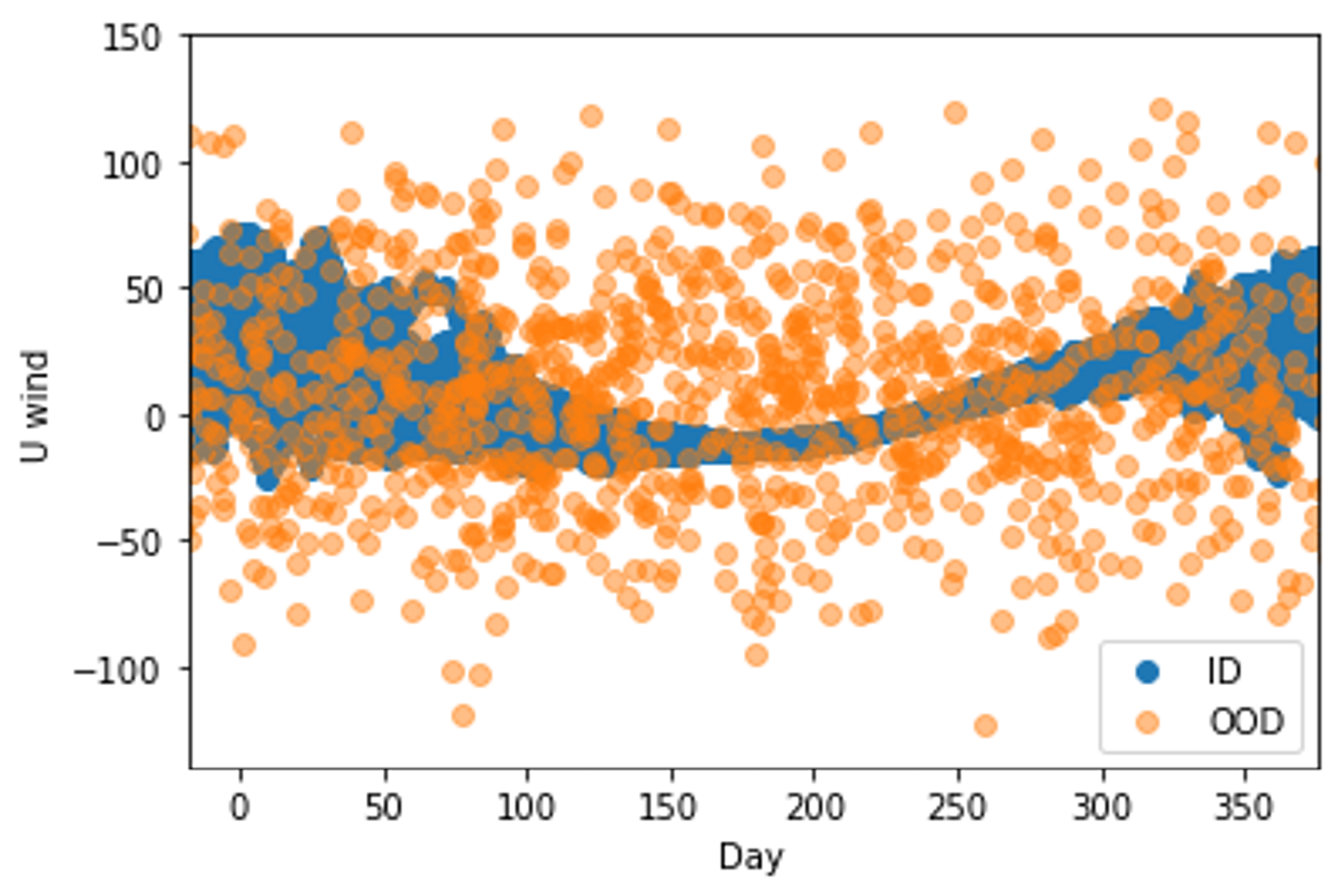}  
    \caption{}
    \label{fig:ood_sde}
\end{subfigure}
\begin{subfigure}{.49\textwidth}
    \centering
    \includegraphics[width=.95\linewidth]{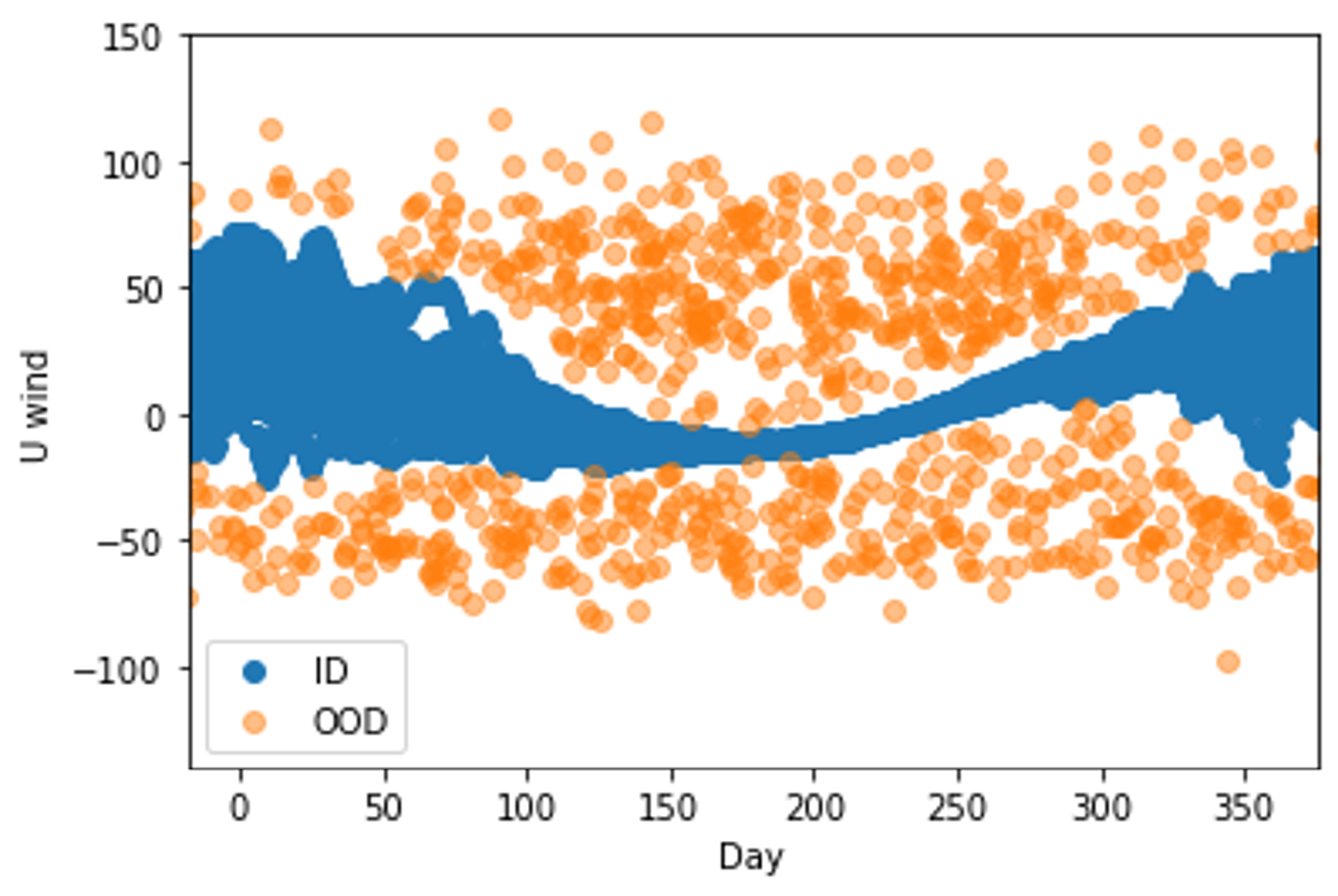}  
    \caption{}
    \label{fig:ood_delaysde}
\end{subfigure}
\caption{OOD data created for  stratospheric wind using (a) Gaussian random noise (b) soft Brownian offset method. The stratospheric wind data is described in Section \ref{subsect:case_study}. The blue points are scattered U wind data over 30 years, while the yellow points are 2000 artificially created OOD data points.}
\label{fig:id_ood}
\end{figure}

\subsubsection{The N steps ahead Delay-SDE-net}
\label{subsubsect:prediction_N_steps_forward}
With small modifications in the training methodology of the Delay-SDE-net presented in Sections\,\ref{subsubsect:train_f}-\ref{subsubsect:train_g_e}, a model predicting $N$ steps forward in time given an initial path can be achieved.

The drift net of the $N$-steps ahead Delay-SDE-net is trained by the optimization
\begin{align*}
\min_{\theta_{f_m}}E_{\mathcal{D}_0}\left[\norm{\hat{\bs{x}}_{t_{k+N}}-\bs{x}_{t_{k+N}}}_2\right],
\end{align*}
where $\hat{\bs{x}}_{t_{k+N}}$ is predicted from $\bs{x}_{t_{k}}$ with a similar approach as in Section\,\ref{subsubsect:train_f}, Eq.\,\eqref{eq:training_values_of_drift_net}. The exact method is illustrated in Algorithm\,\ref{alg:non-image}. Also as in Section\,\ref{subsubsect:train_f}, Eq.\,\ref{eq:optimized_values_of_drift_net}, we obtain an optimal prediction $\hat{\bs{x}}_{t_{k+N}}^{\text{optimal}}$ from the trained network $f_m$ to use in training of the aleatoric diffusion net, $g_{a,m}$.

As in Section\,\ref{subsubsect:train_g_a}, the aleatoric diffusion net is trained using the residuals $\bs{e}_{t_{k+N}} = \hat{\bs{x}}_{t_{k+N}}^{\text{optimal}} - \bs{x}_{t_{k+N}}$. However, the approximated variance for an $N$ steps forward prediction differs from Eq.\,\eqref{eq:theoretical_variance_approximation}. That is, the approximated (time-dependent) variance $N$ steps forward in time is given by
\begin{align*}
    \int_{t_{k}}^{t_{k+N}} g_{a,m}^2(t,\Pi(\eta(t_k)))dt &\simeq  \sum_{i=0}^{N-1}g_{a,m}^2(t_{k+i},\bs{x}_{t_k},\ldots, \bs{x}_{t_{k+1-p}})\Delta t\\
    &\coloneqq V(t_k,\bs{x}_{t_k},\ldots, \bs{x}_{t_{k+1-p}})\Delta t,
\end{align*}
where $V$ is taken as the sum over point variances from time $t_k$ to $t_{k+N}$. Remember that, by definition of the Delay-SDE-net in Section\,\ref{subsect:theoretical_setup}, the model diffusion is state-independent, and therefore only dependent on the initial (observed) path. Further, notice that $V$ only depends on the first time step $t_k$, in addition to observed point values $\bs{x}_{t_k},\ldots, \bs{x}_{t_{k+1-p}}$. This holds in estimation since $t_k$ can be seen as a sufficient estimator for a given time vector $\bs{\tau}^T=[t_k,t_{k+1},\ldots,t_{k+N}]$, as long as $N$ and $\Delta t$ remain constant for all training iterations. This is because, if two observed first time steps are equal $t_a$=$t_b$, then the two time vectors are also equal $\bs{\tau}^T_a=\bs{\tau}^T_b$ in this case, which means $t_k$ fulfils the requirements of a sufficient estimator for $\bs{\tau}^T$ given in \cite{casella2021statistical}, Def.\,6.2.1. Therefore, the N steps ahead variance is trained on the optimization
\begin{align*}
\min_{\theta_{g_{a,m}}}E_{\mathcal{D}_0}\left[\norm{V(t_k,\cdot)\Delta t-(\bs{e}_{t_{k+N}})^2}_2\right].
\end{align*}

Finally, the epistemic diffusion net classifier of the N steps ahead Delay-SDE-net, $g_{e,m}^{\text{prob}}$, is trained exactly as in Section\,\ref{subsubsect:train_g_e}, Eq.\,\ref{eq:loss_epistemic_diffusion_net_prob}, using only the initial time $t_k$ in training. That is, we assume
\begin{align*}
g_{e,m}^{\text{prob}}(t_{k+N}) = g_{e,m}^{\text{prob}}(t_k),
\end{align*} 
because information from ID and OOD training data sets are limited to the initial time $t_k$, meaning we assume that $g_{e,m}^{\text{prob}}$ is constant over the prediction steps from $t_k$ to $t_{k+N}$. The final epistemic diffusion net $g_{e,m}=\sigma_e^{(N)}g_{e,m}^{\text{prob}}$, $\sigma_e^{(N)}\in\R$, of the N steps ahead Delay-SDE-net can be trained with $\sigma_e^{(N)}$ as a tuning parameter on
\begin{align} \label{eq:min_sigma}
\min_{\sigma_e^{(N)}}E_{\mathcal{D}}\left[\norm{(\sigma_e^{(N)} g_{e,m}^{\text{prob}}(t_k,\cdot) + \sqrt{V(t_k,\cdot)\Delta t})^2 - \bs{e}_{t_{k+N}}^2}_2\right],
\end{align}
where $\sqrt{V(t_k,\cdot)\Delta t}$ represents the trained aleatoric standard deviation.

\begin{algorithm}
\caption{Delay-SDE-net. $p$: Number of lags, $N$: Number of  time steps forward to predict, $\pi_s$: Sampled time points in minibatch, $\Delta t$: Length of time step}\label{alg:non-image}
\begin{algorithmic}
\State Initialize $f_m$, $g_{a,m}$ and $g_{e,m}$.
\For{\# $f_m$ training iterations}
\State Sample minibatch  \{$[\bs{x}_{t_k},\ldots ,\bs{x}_{t_{k+1-p}}]^T\}_{k=0}^{n}$, $\bs{x}_{t_k}\in \R^d$ for $t_k\in \pi_s$
\For{all t}
    \State $\hat{\bs{x}}_{t_{k+1}}=\bs{x}_{t_k}+f_m(\bs{x}_{t_k},\ldots ,\bs{x}_{t_{k+1-p}},t_k)\Delta t$
    \State $\hat{\bs{x}}_{t_{k+2}}=\hat{\bs{x}}_{t_{k+1}}+f_m(\hat{\bs{x}}_{t_{k+1}},\ldots,\bs{x}_{t_{k+2-p}},t_{k+1})\Delta t$
    \State \vdots
    \State $\hat{\bs{x}}_{t_{k+N}}=\hat{\bs{x}}_{t_{k+N-1}}+f_m(\hat{\bs{x}}_{t_{k+N-1}},\ldots,\hat{\bs{x}}_{t_{k+N-p}},t_{k+N-1})\Delta t$
\EndFor
\State Update $f$ by:
\State $\nabla_{f_m}L((\bs{x}_{t_{k+N}},\hat{\bs{x}}_{t_{k+N}})_{\forall t_k \in \pi_s})$
\EndFor
\item[]
\For{\# $g_{a,m}$ training iterations}
\State Update $g_{a,m}$ by:  \State$\nabla_{g_{a,m}}L((g_{a,m}^2(\bs{x}_{t_k},\ldots,\bs{x}_{t_{k+1-p}},t_k)\Delta t,e_{t_{k+N}}^2)_{\forall t_k \in \pi_s}), e_{t_{k+N}}=\bs{x}_{t_{k+N}}-\hat{\bs{x}}_{t_{k+N}}$
\EndFor
\item[]
\For{\#  $g_{e,m}$  training iterations}
\State Sample minibatch of OOD data \{$[\tilde{\bs{x}}_{t_k},\ldots ,\tilde{\bs{x}}_{t_{k+1-p}}]^T\}_{k=0}^{n}$, $\tilde{\bs{x}}_{t_k}\in \R^d$ for $t_k\in \tilde{\pi}_s$
\State Update $g_{e,m}$ by:
    \State $\nabla_{g_{e,m}}g_{e,m}(\tilde{\bs{x}}_{t_k},\ldots,\tilde{\bs{x}}_{t_{k+1-p}},t_k)_{\forall t_k \in \tilde{\pi}_s} - \nabla_{g_{e,m}}g_{e,m}(\bs{x}_{t_k},\ldots,\bs{x}_{t_{k+1-p}},t_k)_{\forall t_k \in \pi_s}$
\EndFor
\end{algorithmic}
\end{algorithm}


\section{Theoretical error of the two-layer Delay-SDE-net}
\label{sect:convergence_theory}


In Section\,\ref{subsect:theoretical_setup}, we introduced the Delay-SDE-net (Eq.\,\eqref{eq:continuous_nn_sdde}) and presented its discretized version and real-world counterpart in Eq.\,\eqref{eq:data_generating_sdde} and Eq.\,\eqref{eq:discretized_nn_sdde} respectively. The aim of this section is to derive a theoretical error and convergence result for the two-layer Delay-SDE-net.

As training of the Delay-SDE-net is performed using its discrete-time version, the error result of the model has to be derived through two approximation steps. That is, first we approximate the assumed real-world SDDE model using the continuous-time Delay-SDE-net, then we approximate the continuous-time Delay-SDE-net model using its discretized version. Necessary assumptions on the real-world SDDE and Delay-SDE-net model coefficients follow.

\begin{assum}
\label{assum:sdde_coefficients}
Given an SDDE as in Eq.\,\eqref{eq:data_generating_sdde}, with unique solution, we assume that $f^j,g^j\in \mathcal{B}$, $j=1,\ldots ,d$, are bounded functions on some bounded domain $D\subset T\times \R^{d p}$.
\end{assum}

\begin{assum}
\label{assum:Delay-SDE-net_coefficients}
Given a continuous-time Delay-SDE-net as in Eq.\,\eqref{eq:continuous_nn_sdde}, assume that $\eta\in C$ is of bounded variation and is $(\gamma/2)$-H{\"o}lder continuous, where $0<\gamma \leq 1$. Further, assume that $f^j_m,g^j_m\in C^{1,2}(T\times \R^{d p},\R)$, $j=1,\ldots , d$, are bounded functions having bounded first and second spatial derivatives.
\end{assum}

A convergence result derived and presented in \cite{mohammed01} and \cite{mohammed04} for the Milstein scheme of the dynamics in Eq.\,\eqref{eq:sdde_partitioned}, is used to express the convergence rate between the discretized Delay-SDE-net in Eq.\,\eqref{eq:discretized_nn_sdde} and the continuous-time Delay-SDE-net in Eq.\,\eqref{eq:continuous_nn_sdde}. By \cite{mohammed04}, Thm.\,5.2 and Remark in Sect.\,5.3, and by Assumption\,\ref{assum:Delay-SDE-net_coefficients}, we have that if 
\begin{align*}
    \sup_{-\tau\leq t\leq 0}\norm{\eta^{\pi}(t) - \eta(t)}_2\leq C_\eta\Delta t^\gamma,
\end{align*}
where $C_\eta$ is a positive constant, there exists a constant $C_T>0$, dependent on $T$ and independent of $\Delta t$, such that
\begin{align}
    \label{eq:convergence_rate_sdde}
    \sup_{-\tau \leq t \leq T}\norm{X^{\pi,(m)}(t) - X^{(m)}(t)}_{L_2}^2\leq C_T\Delta t^{2\gamma}.
\end{align}
As we will see in the end of this section (Corollary\,\ref{cor:barron_convergence}), the convergence result in Eq.\,\eqref{eq:convergence_rate_sdde} is a crucial element in obtaining a final error bound for $X^{\pi,(m)}$, compared to the real world model $X$. First, we have to derive an error bound of the continuous-time Delay-SDE-net, $X^{(m)}$. The following theorem gives an error bound of $X^{(m)}$, compared to the real-world SDDE, $X$. We will see that the Barron space theory introduced in Section\,\ref{subsect:nn_barron_space_theory}, together with an upper bound of the epistemic uncertainty, $g_{e,m}$, coming from the training algorithm in Section\,\ref{subsubsect:train_g_e} (Section\,\ref{subsubsect:prediction_N_steps_forward} for N steps ahead prediction), are used to obtain the final result. Finally, note that the result in Theorem\,\ref{thm:continuous_Delay-SDE-net_error} generalizes to a state-dependent diffusion term, compared to the assumed real-world model in Eq.\,\eqref{eq:data_generating_sdde} where the diffusion term is only dependent on the initial path.

\begin{theorem}
    \label{thm:continuous_Delay-SDE-net_error}
    Given an SDDE as in Eq.\,\eqref{eq:sdde_partitioned} and its associated (continuous-time) two-layer Delay-SDE-net, under Assumption\,\ref{assum:sdde_coefficients} and \ref{assum:Delay-SDE-net_coefficients}, the error bound of the model is given by
    \begin{align*}
        \norm{X^{(m)}(t) - X(t)}_{L_2}\leq C_m ,
    \end{align*}
    where 
    \begin{align*}
    C_m \coloneqq \left(4T(C_{g_e} + 2C_B\epsilon_F + \frac{3C_\sigma}{m}(2C_{g}+TC_f))e^{2pC_L(4+2T)t}\right)^{1/2}.
\end{align*}
    Here $m\in\mathbb{N}_+$ is the number of neurons in the two-layer neural networks $f_m$ and $g_{a,m}$ (or the minimum value neuron if the number of neurons is different in the two neural networks), $C_g,C_f>0$ are the upper bounds of $\norm{g}_\mathcal{B}^2$ and $\norm{f}_\mathcal{B}^2$ respectively, $C_{g_e}>0$ is the upper bound of $\norm{g_{e_m}}_{L_2}^2$ (see Sections\,\ref{subsubsect:train_g_e}-\ref{subsubsect:prediction_N_steps_forward}), $C_L>0$ is the Lipschitz constant of $f$ and $g$, $C_B>0$ is the maximum bound of the functions $(g_m-g)^2$ and $(f_m-f)^2$, and $\epsilon_F$ is a small probability representing rare events. Note that we set $\epsilon_F=0$ if the SDDE in Eq.\,\eqref{eq:sdde_partitioned} is defined on a bounded domain. Finally, we have that $C_\sigma = (\gamma(\sigma) + \min\{\abs{\sigma'(+\infty)},\abs{\sigma'(-\infty)}\} + \abs{\sigma(0)})^2$, where $\sigma$ represents the activation function of choice (can be different for $f_m$, $g_{a,m}$ and $g_{e,m}$), and $\gamma(\cdot)$ is given in Definition\,\ref{def:the_gamma-function}.
\end{theorem}

\begin{proof}
First, is it easy to see that
\begin{align}
\label{eq:initial_result}
\sup_{-\tau\leq t \leq 0}\norm{X^{(m)}(t) - X(t)}_{L_2} & = \sup_{-\tau\leq t \leq 0}\norm{\eta(t) - \eta(t)}_{2} = 0.
\end{align}
Using the identity $\abs{u+v}^q\leq 2^{q-1}(\abs{u}^q + \abs{v}^q)$, It{\^o} isometry and Cauchy-Schwarz inequality, we find that
\begin{align}
    \begin{split}
    \label{eq:continuous_nn_convergence_1st_result}
    \norm{X^{(m)}(t) - X(t)}_{L_2}^2 &= E\left[\left(\int_0^t \left(g_m - g\right)dW(s)+\int_0^t \left(f_m - f\right)ds \right)^2\right] \\
    &\leq  2E\left[\left(\int_0^t\left(g_m - g\right)dW(s)\right)^2+\left(\int_0^t \left(f_m - f\right)ds \right)^2\right] \\
    &\leq  2E\left[\int_0^t\left(g_m - g\right)^2ds\right]+2E\left[\int_0^t \left(f_m - f\right)^2ds\int_0^t ds\right] \\
    &\leq 4\int_0^t \norm{g_{e,m}}^2_{L_2}ds + 4\int_0^t \norm{g_{a,m} - g}^2_{L_2}ds + 2a\int_0^t \norm{f_m - f}_{L_2}^2ds
    \end{split}
\end{align}
where in the last step, we used the fact that $g_m=g_{a,m}+g_{e,m}$, the identity $\abs{u+v}^q\leq 2^{q-1}(\abs{u}^q + \abs{v}^q)$, and monotone convergence. Note that we use $f_{m}\coloneqq f_{m}(s,\Pi(X_s^{(m)}))$ and $f\coloneqq f(s,\Pi(X_s))$ for simplicity, and similar for $g_m$, $g$, $g_{a,m}$ and $g_{e,m}$.

Now, let $h_m$ and $h$ represent $g_{a,m}$, $f_m$ and $g$, $f$, respectively. Then we have
\begin{align}
    \begin{split}
    \label{eq:sdde_coefficient_norms}
    &\norm{h_m(t,\Pi(X_t^{(m)})) - h(t,\Pi(X_t))}_{L_2} \leq \\ &\norm{h_m(t,\Pi(X_t^{(m)})) - h(t,\Pi(X_t^{(m)}))}_{L_2} + \norm{h(t,\Pi(X_t^{(m)})) - h(t,\Pi(X_t))}_{L_2},
    \end{split}
\end{align}
by Minkowski inequality.

In the following, each of the two terms on the right hand side of Eq.\,\eqref{eq:sdde_coefficient_norms} will be considered separately.

\begin{enumerate}
    \item Note that $X_t^{(m)}=[X_{t}^{1,(m)},\ldots ,X_{t}^{d,(m)}]^T \in\R^d$, where $X_{t}^{j,(m)}\in C$ for $j=1,\ldots ,d$. That is, $X_t^{(m)}\in C(\Omega\times J;\R^d)$ for every $t\in[0,T]$. Further, $\Pi(X_t^{(m)}) = [X^{(m)}(t+u_1),\ldots ,X^{(m)}(t+u_p)]^T :\Omega\to\R^{d p}$ for fixed $u_1,\ldots ,u_p\in[-\tau,0]$, $\tau>0$, where $X^{(m)}(t+u_i):\Omega\to \R^d$ for $i=1,\ldots ,p$. Now, define the multivariate cumulative distribution function of $\Pi(X_t^{(m)})$ as $F:\R^{d p}\to [0,1]$. Further, split the measurable set of values of $\Pi(X_t^{(m)})$, that is $\R^{d p}$, into a bounded domain, $D_r\coloneqq \{x\in\R^{d p}:\abs{x}<r\}$ for some $r>0$, and its complement $D^c_r$, such that probabilities over $D_r^c$ are small. Then we obtain the following for the first term on the right hand side of Eq.\,\eqref{eq:sdde_coefficient_norms}.

\begin{align*}
    \norm{h_m(t,\Pi(X_t^{(m)})) - h(t,\Pi(X_t^{(m)}))}_{L_2}^2 &= E\left[(h_m(t,x)-h(t,x))^2\right] \\
     & = \int_{\R^{dp}}(h_m(t,x)-h(t,x))^2dF(x) \\
     & = \int_{D_r}(h_m(t,x)-h(t,x))^2dF(x) +\\ &\quad\quad\int_{D^c_r}(h_m(t,x)-h(t,x))^2dF(x).
\end{align*}
By Assumption\,\ref{assum:sdde_coefficients} and \ref{assum:Delay-SDE-net_coefficients} we know that $h_m$ and $h$ are bounded functions. Using this, and the fact that $D_r^c$ is a domain of rare events, we have that
\begin{align*}
    \int_{D^c_r}(h_m(t,x)-h(t,x))^2dF(x) \leq C_BF(D_r^c)\leq C_B\epsilon_F,
\end{align*}
for all $t\in [0,T]$, where $C_B>0$ is a constant and $\epsilon_F>0$ represents the small probability over $D^c_r$. Further, by Assumption\,\ref{assum:sdde_coefficients} we have $h\in\mathcal{B}$, and therefore, by Theorem\,\ref{thm:convergence_barron_functions} we find 
\begin{align*}
    \int_{D_r}(h_m(t,x)-h(t,x))^2dF(x) \leq \frac{3C_\sigma \norm{h}_{\mathcal{B}}^2}{m}.
\end{align*}
    \item By Lipschitz continuity, and by definition of the Euclidean norm, the projection $\Pi$ (Eq.\,\eqref{eq:the_projection}) and the segment process $X_t$ (Eq.\,\eqref{eq:the_segment_process}), the second term on the right hand side of Eq.\,\eqref{eq:sdde_coefficient_norms} is given by 
\begin{align*}
    \norm{h(t,\Pi(X_t^{(m)})) - h(t,\Pi(X_t))}_{L_2}^2 & = E\left[\left(h(t,\Pi(X_t^{(m)})) - h(t,\Pi(X_t))\right)^2\right] \\
    &\leq C_L E\left[\norm{\Pi(X_t^{(m)}) - \Pi(X_t)}_2^2\right] \\
    &= C_L E\left[\sum_{i=1}^{p}\left(X^{(m)}_t(u_i) - X_t(u_i)\right)^2\right] \\
    &= C_L \sum_{i=1}^{p}\norm{X^{(m)}_t(u_i) - X_t(u_i)}_{L_2}^2.
\end{align*}
Since the norm is increasing as $t+u_i$ grows (see Eq.\,\eqref{eq:the_segment_process}) and $u_i\in[-\tau,0]$, and by use of Eq.\,\eqref{eq:initial_result}, we further find that
\begin{align*}
    &\norm{h(t,\Pi(X_t^{(m)})) - h(t,\Pi(X_t))}_{L_2}^2 \\
    &\leq C_L\sum_{i=1}^{p}\left(\sup_{0\leq s \leq -u_i}\norm{X^{(m)}_s(u_i) - X_s(u_i)}_{L_2}^2 + \sup_{-u_i\leq s \leq t-u_i}\norm{X^{(m)}_s(u_i) - X_s(u_i)}_{L_2}^2\right) \\
    &\leq C_L\sum_{i=1}^{p}\norm{X^{(m)}(t) - X(t)}_{L_2}^2 = pC_L\norm{X^{(m)}(t) - X(t)}_{L_2}^2.
\end{align*}
\end{enumerate}

Applying the identity $\abs{u+v}^q\leq 2^{q-1}(\abs{u}^q + \abs{v}^q)$ to Eq.\,\eqref{eq:sdde_coefficient_norms} with the results from points 1. and 2. above inserted, we find that
\begin{align}
    \label{eq:hm_minus_h_result}
    \norm{h_m(t,\Pi(X_t^{(m)})) - h(t,\Pi(X_t))}_{L_2}^2 \leq 2\left(\xi(h) + pK\norm{X^{(m)}(t) - X(t)}_{L_2}^2\right),
\end{align}
where we have defined $\xi(h)\coloneqq \frac{3C_\sigma \norm{h}^2_{\mathcal{B}}}{m} + C_B\epsilon_F$.

From Sections\,\ref{subsubsect:train_g_e}-\ref{subsubsect:prediction_N_steps_forward}, we know how to find an upper bound for the epistemic error $\norm{g_{e,m}}_{L_2}^2\leq C_{g_e}$. Insert this, as well as the result in Eq.\,\eqref{eq:hm_minus_h_result}, into Eq.\,\eqref{eq:continuous_nn_convergence_1st_result} to find
\begin{align*}
    \norm{X^{(m)}(t) - X(t)}_{L_2}^2  \leq& 4\int_0^TC_{g_e}ds + 8\int_0^T\xi(g)ds + 4T\int_0^T\xi(f)ds \\
    &+8pC_L\int_0^t\norm{X^{(m)}(s) - X(s)}_{L_2}^2ds  +4TpC_L\int_0^t\norm{X^{(m)}(s) - X(s)}_{L_2}^2ds \\
    =& 4T(C_{g_e} + 2\xi(g) + T\xi(f)) + 2pC_L(4+2T)\int_0^t\norm{X^{(m)}(s) - X(s)}_{L_2}^2ds.
\end{align*}

Finally, define the constants $C_1\coloneqq 4T(C_{g_e} + 2\xi(g) + T\xi(f))$ and $C_2\coloneqq 2pC_L(4+2T)$, as well as the function $G(t)\coloneqq \norm{X^{(m)}(t) - X(t)}_{L_2}^2$, and find by Gr{\"o}nwall's inequality that 
\begin{align*}
    G(t) \leq C_1 + C_2\int_0^t G(s) ds \quad\leftrightarrow\quad  G(t) \leq C_1\exp\left(C_2t\right)
\end{align*}
This concludes the proof.
\end{proof}

When estimating a Delay-SDE-net to imitate a real-world SDDE representing some dynamical system, the estimation procedure takes place in the discrete-time domain. That is, the upper bound of the total error is larger than the upper bound stated in Theorem\,\ref{thm:continuous_Delay-SDE-net_error}. The following corollary gives a theoretical upper bound of the total error of the estimated Delay-SDE-net.

\begin{cor}
\label{cor:barron_convergence}
Given an SDDE as in Eq.\,\eqref{eq:sdde_partitioned} and an associated (continuous-time) two-layer Delay-SDE-net, under Assumption\,\ref{assum:sdde_coefficients} and \ref{assum:Delay-SDE-net_coefficients}, the error bound of the estimated $\pi$-discretized version of the two-layer Delay-SDE-net ($\pi$ defined from Eq.\,\eqref{discretization_partition}) is given by
\begin{align*}
    \norm{X^{\pi,(m)}(t) - X(t)}_{L_2} \leq C^{1/2}_T\Delta t^{\gamma} + C_m.
\end{align*}
The constant $C_m$ is given in Theorem\,\ref{thm:continuous_Delay-SDE-net_error}. Further, $C_T>0$ is a constant dependent on $T$ and independent of $\Delta t$, and $\gamma$ is a regularity condition on the initial path of the real-world SDDE (see Assumption\,\ref{assum:Delay-SDE-net_coefficients}).
\end{cor}
\begin{proof}
By Minkowski inequality we find
\begin{align}
    \begin{split}
    \label{eq:cor_study_convergence_help}
    \norm{X^{\pi,(m)}(t) - X(t)}_{L_2} & = \norm{X^{\pi,(m)}(t) - X^{(m)}(t) + X^{(m)}(t) - X(t)}_{L_2}\\
    &\leq \norm{X^{\pi,(m)}(t) - X^{(m)}(t)}_{L_2} + \norm{X^{(m)}(t) - X(t)}_{L_2}.
    \end{split}
\end{align}
Inserting results from Eq.\,\eqref{eq:convergence_rate_sdde} and Theorem\,\ref{thm:continuous_Delay-SDE-net_error} completes the proof.
\end{proof}


\section{Analysis and application}
\label{sect:Analyses_and_application}


This section presents numerical and empirical studies of the two-layer Delay-SDE-net. This includes a simulation study in Section \ref{subsubsect:numerical_convergence_study} to explore numerical convergence of the Delay-SDE-net, as well as a comparison study in Section \ref{subsubsection:sim_compare} where the Delay-SDE-net is compared with two benchmark models on simulated data. Finally, we evaluate the performance of the Delay-SDE-net on a real world case study in Section \ref{subsect:case_study} involving weather data. 

\subsection{Simulation study}
\label{subsect:simulation_study}

A simulation study is performed to analyze the characteristics and performance of the two-layer Delay-SDE-net. That is, an SDDE representing the real-world model, see Eq.\,\eqref{eq:data_generating_sdde}, is defined to simulate data sets for training and evaluating the Delay-SDE-net. The numerical convergence rate of the Delay-SDE-net model is studied by testing a trained reference solution of the Delay-SDE-net model on data with different time resolutions. Further, the Delay-SDE-net is compared to two other time series models by assessing the prediction error of drift and diffusion coefficients.

\subsubsection{Simulation setup}
\label{subsubsect:simulation_study_setup}

Consider a real-world SDDE as in Eq.\,\eqref{eq:data_generating_sdde}, where $d=2$ and $p=p_1=p_2=4$ (meaning $\Pi = \Pi_1=\Pi_2$). That is, we assume that a real-world dynamical system evolves in time like $X(t)\in\R^{2}$, and is dependent on $\Pi(X_t(u))\in\R^{8}$, where $t\in[0,T]$, $u\in[-\tau,0)$, $T,\tau\geq 0$. Note that the point dependencies from the projection, $\Pi$, depends on the defined discretization scheme. For example, when $\Delta t = 1$ we have $(u_1,u_2,u_3,u_4) = (-4,-3,-2,-1)$ (see further details in Section\,\ref{subsubsect:numerical_convergence_study}). We define the SDDE coefficients with parameters $(a_{(i)}^j,w_{(i)}^j,b^j)\in \R\times \R^{9} \times \R$ as follows. Let the drift $f = [f^1,f^2]^T$ be given as
\begin{align*}
    f^j(t,x)\coloneqq f^j(t,\Pi(X_t))&=\sum_{i=1}^{2} a_i^j \sigma_1((w_i^j)^T [t,x]^T),\quad j=1,2,
\end{align*}
where $a^1_1=a^2_1 = a^1_2=a^2_2 = 5$, the weights $w_i^j\in \R^9$, $i=1,2$, $j=1,2$, are given in Eq.\,\eqref{eq:example_weights}, and 
\begin{align*}
    \sigma_1(z) = \frac{e^{\alpha z}-e^{-\alpha z}}{e^{\alpha z}+e^{-\alpha z}},
\end{align*}
is the tanh activation function. Further, the diffusion $g=[g^1,g^2]^T$ is given by
\begin{align*}
    g^j(t,x)\coloneqq g^j(t,\Pi(\eta(0))) = a^j\sigma_2((w^j)^T [t,x]^T +b^j),\quad j=1,2,
\end{align*}
where $a^1=4$, $a^2 = 1/8$, the weights $w^j\in \R^9$, $j=1,2$, are given in Eq.\,\eqref{eq:example_weights}, $b^1=0$, $b^2 = 1$, and
\begin{align*}
    \sigma_2(z) = \frac{1}{1 + e^{-\lambda z}},
\end{align*}
is the sigmoid activation function. Note that $f^1,f^2,g^1,g^2\in\mathcal{B}$ by definition (see Section\,\ref{subsect:nn_barron_space_theory}).

Sample paths on $t\in[0,T]$ are simulated using the discretization scheme 
\begin{align}
    \label{eq:simulation_X}
    X^{j,\pi}(t_{k+1}) = & X^{j,\pi}(t_k) + f^{j}(t_k,\Pi(X_{t_k}))\Delta t +g^{j}(t_k,\Pi(\eta(0)))\sqrt{\Delta t} Z^j(t_k),
\end{align}
for each $j=1,2$, where $Z^j(t_k)\sim N(0,1)$, and the weights of $f$ and $g$ are given by
\begin{align}
\label{eq:example_weights}
\begin{bmatrix}
(w_1^1)^T\\
(w_2^1)^T\\
(w^2_1)^T\\
(w_2^2)^T\\
(w^1)^T\\
(w^2)^T
\end{bmatrix}=
\begin{bmatrix}
0 & 3 & 2 & 2 & 5 & -3 & 1 & -3 & -1\\
0 & 1 & 0 & -0.5&0&-1&0&-0.5&0\\
0 & 0&2&0&-3&0&1&0&0\\
0 & 0& 1&0&-0.5&0&0&0&-0.5 \\
-5/365 & 0&0&0&0&0&0&0&0\\
0&0& 0&0&0&0&0&0&1
\end{bmatrix}10^{-2}.
\end{align}
Each of the simulated paths are initiated with $\eta(t)=[\sin (z_1t),\cos (z_2t)]$, $z_1,z_2\sim N(0,1)$, on $t\in [-\tau,0]$. An example of a simulated sample path with $\Delta t = 1$, $T=365$ and $\tau = 4$ is shown in Figure\,\ref{fig:ood_sim} (both as ID and OOD data).

\subsubsection{Numerical convergence study}
\label{subsubsect:numerical_convergence_study}

In the current section we study the convergence rate of a two-layer Delay-SDE-net numerically. The presented methodology is inspired by the work in \cite{lord_powell_shardlow_2014}, Sect.\,8.5.

From Eq.\,\eqref{eq:cor_study_convergence_help} and Corollary \,\ref{cor:barron_convergence} we know that the $L_2$-error of the discrete time two-layer Delay-SDE-net, $X^{\pi,(m)}$, for some fixed time $t\in[0,T]$ is bounded as 
\begin{align}
\begin{split}
\label{eq:convergence_rate_L2_norm_dt_Cm}
    \norm{X^{\pi,(m)}(t) - X(t)}_{L_2} &\leq \norm{X^{\pi,(m)}(t) - X^{(m)}(t)}_{L_2} + \norm{X^{(m)}(t) - X(t)}_{L_2} \\
    &\leq C^{1/2}_T\Delta t^{\gamma} + C_m,
\end{split}
\end{align}
where $X$ represents the solution of the real-world SDDE in Eq.\,\eqref{eq:data_generating_sdde}, and $X^{(m)}$ is the continuous time Delay-SDE-net in Eq.\,\eqref{eq:continuous_nn_sdde}. The numerical estimate of the $L_2$-error of $X^{\pi,(m)}$ is denoted by 
\begin{align*}
    \Delta_{X^{\pi,(m)}}\simeq \norm{X^{\pi,(m)}(t) - X(t)}_{L_2}.
\end{align*}
Remember that $C_m$ represents the upper bound of the $L_2$-error of the continuous time Delay-SDE-net, $X^{(m)}$. That is, $C_m$ is the bound
\begin{align*}
    \norm{X^{(m)}(t) - X(t)}_{L_2}\leq C_m ,
\end{align*}
that is explicitly given in Theorem\,\ref{thm:continuous_Delay-SDE-net_error}. The numerical estimate of the $L_2$-error of $X^{(m)}$ is denoted as
\begin{align*}
    \Delta_{X^{(m)}}\simeq \norm{X^{(m)}(t) - X(t)}_{L_2},
\end{align*}
in the following.

By Eq.\,\eqref{eq:convergence_rate_L2_norm_dt_Cm}, the convergence rate of $X^{\pi,(m)}$ can be studied using the relation
\begin{align}
    \label{eq:study_convergence_theoritical}
    \ln\left(\norm{X^{\pi,(m)}(t) - X(t)}_{L_2} - C_m\right)\leq \frac{\ln C_T}{2} + \gamma \ln\Delta t.
\end{align}
As $C_m$ is an upper bound, and since we know the explicit form of the real-world SDDE, $X(t)$, in this simulation study, we will use the estimated $L_2$-error $\Delta_{X^{(m)}}$ in place of $C_m$ to assess the convergence rate, $\gamma$. Further, the term $\ln C_T/2$ in Eq.\,\eqref{eq:study_convergence_theoritical} does not provide useful information when studying the convergence rate, as it influences the y-axis intercept only. That is, to study the convergence rate of $X^{\pi,(m)}$, we will consider the relation
\begin{align}
\label{eq:numerical_convergence_rate_study}
    \ln\left(\Delta_{X^{\pi,(m)}} - \Delta_{X^{(m)}}\right)\leq \gamma \ln\Delta t.
\end{align}

A step-by-step methodology for how to perform the numerical convergence study for the two-layer Delay-SDE-net follows: \\
\begin{enumerate}
    \item Simulate data from known real-world SDDE, $X(t)$, with small $\Delta t_{\text{ref}}$.
    \item Fit a two-layer Delay-SDE-net to simulated data with resolution $\Delta t_{\text{ref}}$. This is considered an accurate reference solution of $X^{(m)}(t)$.
    \item Compute the estimated $L_2$-error $\Delta_{X^{(m)}}$.
    \item Compute the estimated $L_2$-error $\Delta_{X^{\pi,(m)}}$ using resolution $\Delta t = \kappa\Delta t_{\text{ref}}$, $\mathbb{N}\ni\kappa > 1$, for several $\kappa$, on simulated data from point 1.
    \item Perform a common numerical convergence study to assess the convergence rate, $\gamma$, given in Corollary\,\ref{cor:barron_convergence}.
\end{enumerate}

The methodology above is used to study the convergence rate of a two-layer Delay-SDE-net fit to data simulated by the real-world SDDE defined in Section\,\ref{subsubsect:simulation_study_setup}. That is, we use Eq.\,\eqref{eq:simulation_X} to simulate $M = 1000$ samples of the of the real-world SDDE with 
\begin{align*}
    \Delta t_{\text{ref}} \coloneqq t_{k+1}-t_k = 0.01, \quad t_k\in(0,5], \quad k=1,\ldots ,N_{\text{ref}},
\end{align*}
where $N_{\text{ref}} = 500$, and $t_{-l}\in[-15,0]$, $l=0,\ldots ,1500$, for the initial path $\eta$. Note that this means that $\tau=15$ and $T=5$. The $M$ simulated samples are divided into training/test subsets ($M_{\text{train}}=700/M_{\text{test}}=300$ split). The reference solution of the continuous time Delay-SDE-net is fit using all the points in the training data set, meaning it is fit with resolution $\Delta t_{\text{ref}}$. This gives reference solution $X^{0.01,(m)}$ for the continuous-time Delay-SDE-net $X^{(m)}$. Note that, when predicting $X$ using $X^{0.01,(m)}$, it is important to use the same Brownian path as used to generate data from the real-world process $X$. The computed value of $\Delta_{X^{(m)}}$ is reported in Table\,\ref{tab:numerical_convergence}. See Appendix\,\ref{app:convergence_study_ref_solution} for a more detailed description on the numerics.

In step 4., we make five new predicted paths of the simulated test data set, but now by using the trained Delay-SDE-net (the reference solution) on test subsets with reduced resolutions 
\begin{align*}
\Delta t = \kappa\Delta t_{\text{ref}}, \quad\kappa=5,10,50,100,500.    
\end{align*}
These respective resolutions gives sample paths of sizes $N_\kappa=N_{\text{ref}}/\kappa = 100,50,10,5,1$. Note that the prediction from each resolution represents a discretized Delay-SDE-net model $X^{\pi,(m)}$. That is, $X^{0.05,(m)}$, $X^{0.1,(m)}$, $X^{0.5,(m)}$, $X^{1,(m)}$ and $X^{5,(m)}$ respectively. The corresponding $L_2$-error estimate, $\Delta_{X^{\pi,(m)}}$, is computed for each of the five discretized Delay-SDE-nets, and the results are reported in Table\,\ref{tab:numerical_convergence}. See Appendix\,\ref{app:convergence_study_discretized_models} for more details about the numerics.

\begin{table}[hbt!]
    \centering
    \begin{tabular}{c|c|c|c|c|c|c}
       $\kappa$  & $1$ & $5$ & $10$ & $50$ & $100$ & $500$\\ \toprule
       $\Delta_{X^{(m)}}$ & $1.956$ & $-$ & $-$ & $-$ & $-$ & $-$ \\ 
       $\Delta_{X^{\pi,(m)}}$ & $-$ & $6.030$ & $9.764$ & $26.082$ & $37.658$ & $78.328$ \\ \bottomrule
    \end{tabular}
    \caption{Error estimate of the continuous-time Delay-SDE-net, $\Delta_{X^{(m)}}$, and error estimates of the discrete-time Delay-SDE-net, $\Delta_{X^{\pi,(m)}}$, for five different resolutions.}
    \label{tab:numerical_convergence}
\end{table}

Settings for training of the neural networks in this numerical study can be found in Appendix \ref{a:num_param}.

As expected, we see from Table\,\ref{tab:numerical_convergence} that better resolution of the discretization scheme provides more accurate predictions. It is also apparent that $X^{0.01,(m)}$ (that is fit as an accurate reference solution for the continuous-time Delay-SDE-net) performs considerably better than the corresponding discretized Delay-SDE-nets. Figures\,\ref{fig:num_sim_1}-\ref{fig:num_sim_4} illustrate these facts for the respective models $X^{0.01,(m)}\simeq X^{(m)}$, $X^{0.05,(m)}$, $X^{0.1,(m)}$ and $X^{0.5,(m)}$, where the prediction for one given sample path is graphed for each of these models. That is, the paths predicted using the Delay-SDE-net models tend to deviate with greater distance from the true path for coarser resolution.

\begin{figure}[hbt!]
\centering
\begin{subfigure}{.49\textwidth}
    \centering
    \includegraphics[width=.95\linewidth]{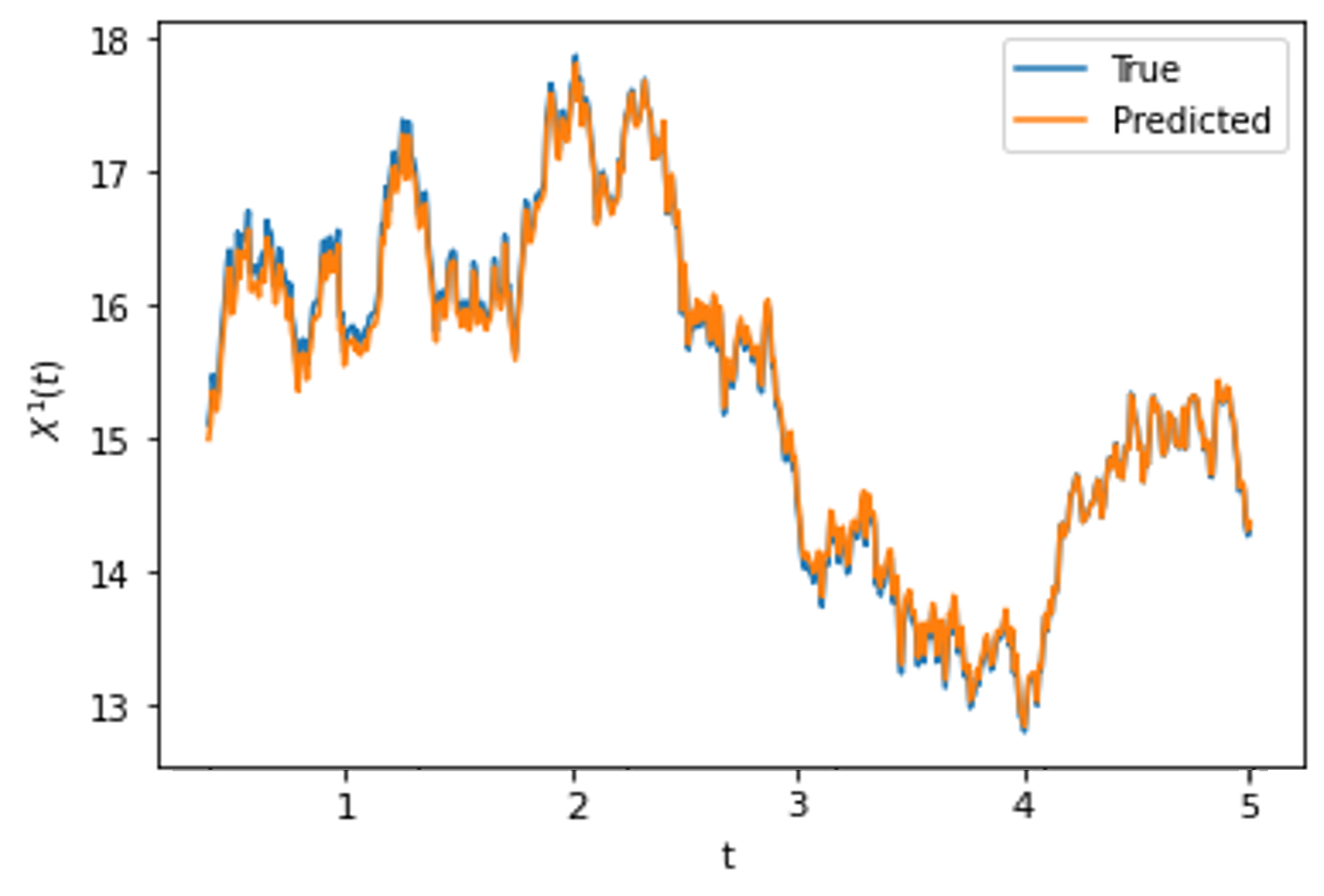}  
    \caption{$\kappa=1$}
    \label{fig:num_sim_1}
\end{subfigure}
\begin{subfigure}{.49\textwidth}
    \centering
    \includegraphics[width=.95\linewidth]{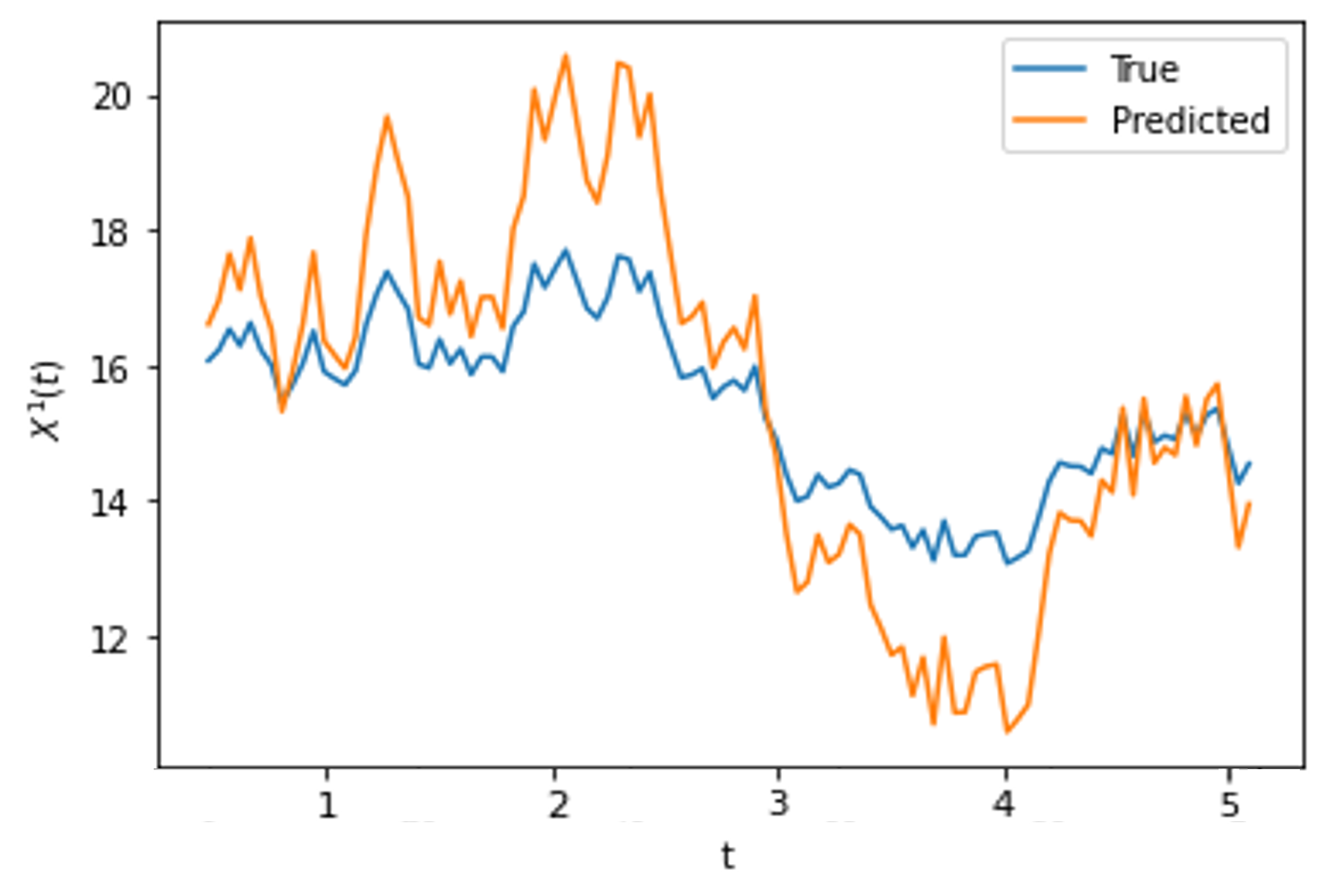}  
    \caption{$\kappa=5$}
    \label{fig:num_sim_2}
\end{subfigure}
\begin{subfigure}{.49\textwidth}
    \centering
    \includegraphics[width=.95\linewidth]{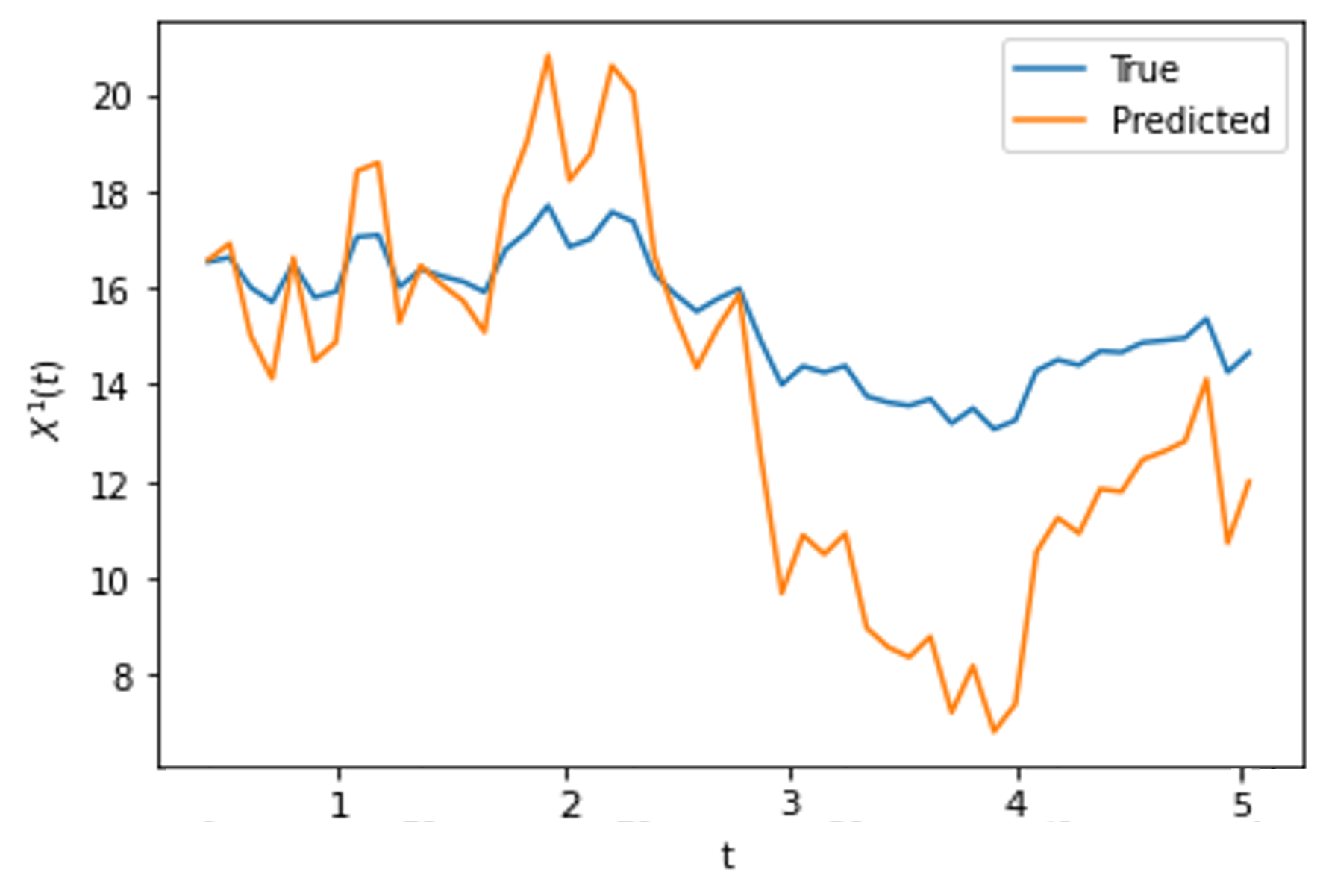}  
    \caption{$\kappa=10$}
    \label{fig:num_sim_3}
\end{subfigure}
\begin{subfigure}{.49\textwidth}
    \centering
    \includegraphics[width=.95\linewidth]{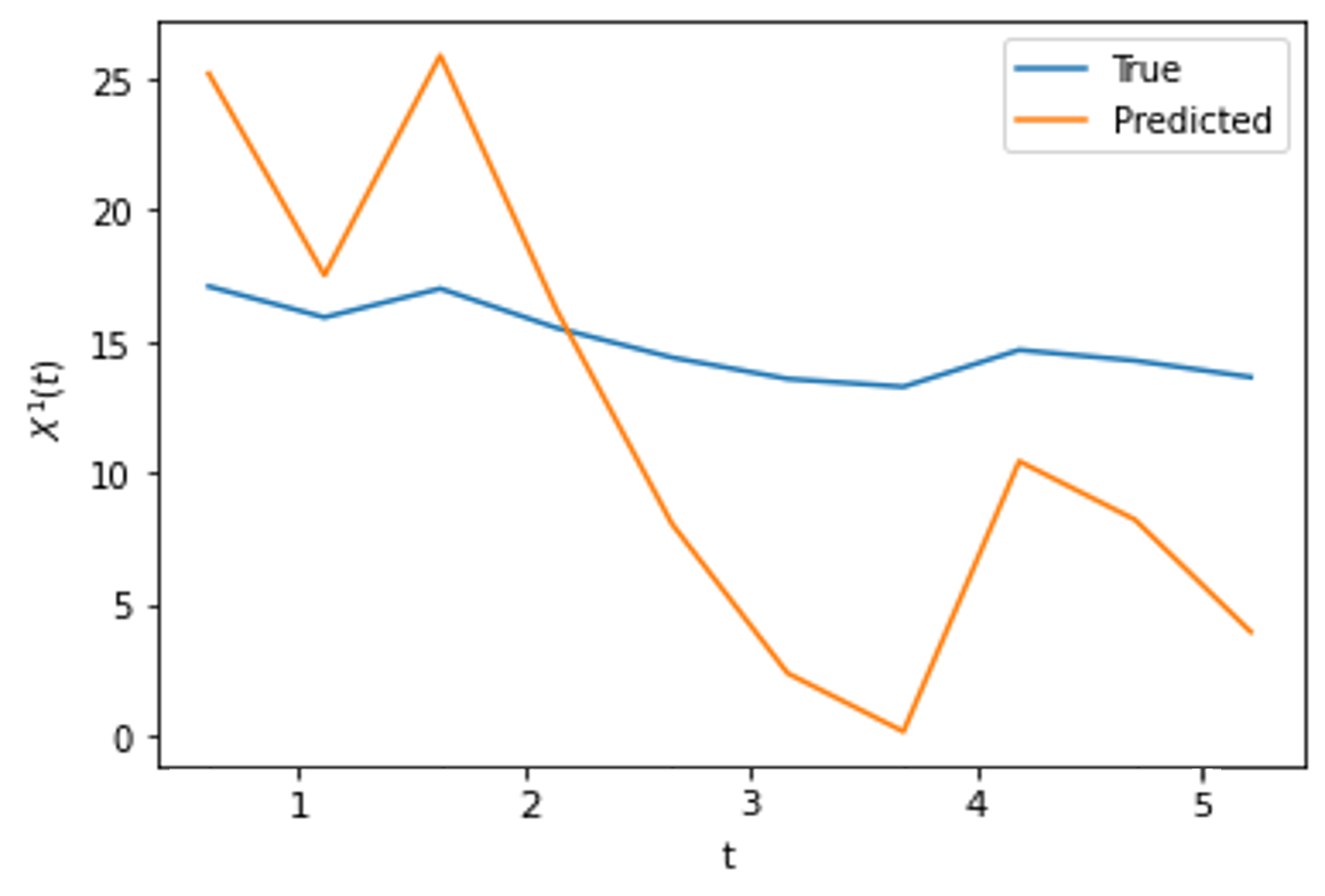}  
    \caption{$\kappa=50$}
    \label{fig:num_sim_4}
\end{subfigure}
\caption{Predictions of a simulated sample path of $X^1(t)$ (Section\,\ref{subsubsect:simulation_study_setup}) made by the Delay-SDE-net models $X^{(m)}(t)$ ($\kappa=1$), $X^{0.05,(m)}(t)$ ($\kappa=5$), $X^{0.1,(m)}(t)$ ($\kappa=10$) and $X^{0.5,(m)}(t)$ ($\kappa=50$).}
\label{fig:num_sim}
\end{figure} 

\begin{figure}[hbt!]
    \centering
    \includegraphics[width=7cm]{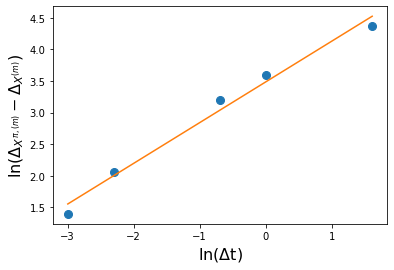}
    \caption{Log-log plot of the approximated $L_2$-errors of the discretized Delay-SDE-net against $\Delta t$. The Delay-SDE-nets represented are $X^{\pi,(m)}(t)$ with discretization step sizes $0.05$, $0.1$, $0.5$, $1$ and $5$. A plot of the corresponding linear regression representing the rate of convergence of $X^{(m)}(t)$ is shown as well.}
    \label{fig:numerical_error}
\end{figure}

The convergence rate of the Delay-SDE-net, $\gamma$, can be assessed by using the values in Table\,\ref{tab:numerical_convergence} as input to Eq.\,\eqref{eq:numerical_convergence_rate_study}. The results are given in Figure\,\ref{fig:numerical_error}, with a corresponding linear regression. The rate of convergence, corresponding to the slope of the linear regression in Figure\,\ref{fig:numerical_error}, is given by $\gamma = 0.6359$ in this study. The estimated convergence rate is better than the theoretical Euler-Maruyama scheme convergence rate, that is $\gamma_{\text{Euler}}=1/2$. Note that the Euler-Maruyama scheme is derived for It\^{o} diffusions with state-dependent drift and diffusion terms. In this work we assume state-independence for the diffusion term of the Delay-SDE-net (see Section\,\ref{subsect:theoretical_setup}), meaning that the corresponding SDDE model is less complex in that respect. However, the SDDE model in this work is more complex than the standard It\^{o} diffusions, as the latter lacks memory. These positive and negative trade-offs between the standard It\^{o} diffusions and our SDDE modelling framework seem to give our model a slight advantage for the rate of convergence.

Note that we compare the Delay-SDE-net convergence rate with the Euler-Maruyama scheme instead of the Milstein scheme, even though the convergence result in Corollary\,\ref{cor:barron_convergence} is derived using the Milstein scheme. The reason for doing this is that the Milstein scheme reduces to the Euler-Maruyama scheme for models with state-independent diffusion.

\subsubsection{Comparison with other models}
\label{subsubsection:sim_compare}

In this section the performance of the Delay-SDE-net is assessed by comparing it to two similar models for dynamical systems. These benchmark models, namely the VAR model and the SDE-net, have similar properties as the Delay-SDE-net. The VAR model is able to capture relations between lagged versions of the dynamical system, but only linearly, unlike the Delay-SDE-net. Both the SDE-net and Delay-SDE-net are based on stochastic differential equations with nonlinear coefficients, however, the SDE-net contain only information about the current state of the system, while the Delay-SDE-net additionally captures nonlinearity in lagged states of the dynamical system. The models are tested on data simulated by the SDDE introduced in Section\,\ref{subsubsect:simulation_study_setup}.

The following points specify how training and test data sets are generated using the simulation setup in Section\,\ref{subsubsect:simulation_study_setup}, and further how the comparison study is performed:\\

\begin{enumerate}
    \item Simulate 110 paths of data with $\Delta t \coloneqq t_{k+1}-t_k = 1, \quad t_k\in[1,365], \quad k=1,\ldots ,365$. The simulated data correspond to daily values over 110 years (where the leap years are removed).
    \item Divide the data into training/validation/test subsets (90/10/10 split). 
    \item Train the prediction models $\hat{f}$ and $\hat{g}_a$ on the training data $\mathcal{D}_{0}$, tune hyper parameters on the validation data and evaluate performance on the test data $\mathcal{D}$. All these data sets are generated by $X^\pi=[X^{1,\pi},X^{2,\pi}]^T$, see Eq.\,\eqref{eq:simulation_X}.
    \item Draw OOD data from a data set $\tilde{\mathcal{D}}$ generated by $\tilde{X}^\pi=[\tilde{X}^{1,\pi},\tilde{X}^{2,\pi}]^T$, being the SDDE in Eq.\,\eqref{eq:simulation_X}, but with increased noise $\tilde{g}^1=2.5g^1$, and $\tilde{g}^2=2.5g^2$. That is, create OOD data with more extreme values than the ID data, that are still realistic. A plot showing one year of simulated ID data and OOD data is given in Figure \ref{fig:ood_sim}.
    \item Replace randomly given intervals of values in test set $\mathcal{D}$ with intervals of values from $\tilde{\mathcal{D}}$. Details about the intervals we replaced are given in Appendix \ref{a:ood_data}. The intervals were drawn randomly to satisfy:
    \begin{align*}
     \mathrm{max}(X^{j,\pi})<\tilde{X}^{j,\pi}<\mathrm{min}(X^{j,\pi}),\quad j=1,2, 
    \end{align*}
    with $X^{j,\pi}\in\mathcal{D}_0$, to make sure that the data points are OOD from the training set. The modified data set $\mathcal{D}^*$  consists of 3.9\% OOD data, which mimics a real world situation where the extreme values are rare. 
    \item Evaluate the performance of $\hat{g}_e^{\text{prob}}$ on identifying OOD data in $\mathcal{D}^*$.
\end{enumerate}

\begin{figure}[hbt!]
\centering
\begin{subfigure}{.49\textwidth}
    \centering
    \includegraphics[width=.95\linewidth]{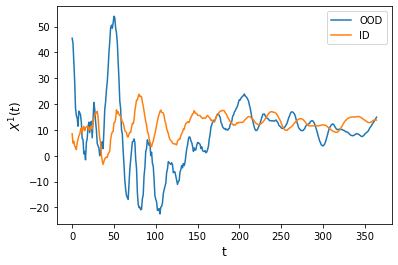}  
    \caption{}
    \label{fig:ood_sim_1}
\end{subfigure}
\begin{subfigure}{.49\textwidth}
    \centering
    \includegraphics[width=.95\linewidth]{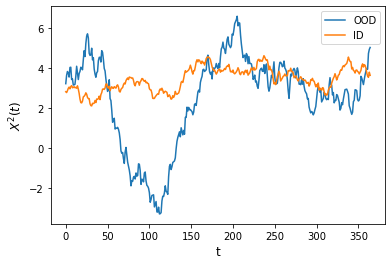}  
    \caption{}
    \label{fig:ood_sim_2}
\end{subfigure}
\caption{One year simulated ID and OOD data for (a) $X^1(t)$ and (b) $X^2(t)$.}
\label{fig:ood_sim}
\end{figure}

The SDE-net is implemented according to the regression example in \cite{kong20}, and the VAR model prediction is based on the empirical work in \cite{eggen2022}. The final tuning parameters for the SDE-net and Delay-SDE-net, as well as the parameters for the VAR model, are given in Appendix \ref{a:comparison_param} and \ref{a:var_param}, respectively.

The performance on simulated data are given in Table\,\ref{tab:sim_results} for the three models. The drift term is evaluated in terms of the root mean squared error (RMSE) between the true $X^1$ and predicted value $\hat{X}^1$ (using Eq.\,\eqref{eq:training_values_of_drift_net}), assuming $X^1$ is the variable of interest and $X^2$ is an explanatory variable. The aleatoric diffusion term is evaluated as the RMSE between the true and predicted aleatoric variance, $g_a^2$ and $\hat{g}_a^2$ respectively. Finally, the epistemic diffusion term is evaluated as the area under the receiver operating curve (ROCAUC) for $\hat{g}_e^{\text{prob}}$ identifying if data is taken from the OOD distribution $\tilde{X}$ (binary classification).

The Delay-SDE-net consistently outperforms the comparing models when predicting $f$, $g_a$ and $g_e^{\text{prob}}$.  It is clear that the Delay-SDE-net has some advantages over the SDE-net and VAR model on this kind of time series data. The VAR model performs better than the SDE-net in predicting $f$ and $g_a$. As a well known fact, and as reflected in the results in Table\,\ref{tab:sim_results}, linear models often perform well in modelling of nonlinear systems, especially for simpler systems with low dimensionality, which corresponds to the simulated data in this study.  If we were to simulate data with more complex structures and interaction effects with higher dimensionality, we expect neural network models to achieve even higher performance compared to the VAR model, than achieved in this simulation study.

The Delay-SDE-net also provides better estimates of $g_e^{\text{prob}}$ than the SDE-net, giving almost a perfect prediction close to 1.  Note that $g_e^{\text{prob}}$ is independent of the number of prediction steps, as it only depends on the input value at the first step. Finally, unlike the Delay-SDE-net and SDE-net, the VAR model does not provide predictions of $g_e^{\text{prob}}$.

\begin{table}[hbt!]
    \centering
    {\tabcolsep5pt
    \begin{tabular}{c|c|c|c|c|c|c|c|c}
          &  \multicolumn{3}{c|}{Value ($x|f$)} &  \multicolumn{3}{c|}{Aleatoric Uncertainity ($g^2_a$)} &  \multicolumn{2}{c}{Epistemic Uncertainity ($g_e)$}\\ 
         &  \multicolumn{3}{c|}{(RMSE)} &  \multicolumn{3}{c|}{(RMSE)} &  \multicolumn{2}{c}{(ROCAUC)}\\ \toprule
         &  VAR & SDE-net & Delay & VAR & SDE-net & Delay & SDE-net & Delay\\ 
         $t+1$  &  0.70 & 1.31  & \textbf{0.66} & 0.58 & 3.03 & \textbf{0.26} &   0.901 & \textbf{0.992} \\
         $t+2$  & 1.09 & 1.82 & \textbf{1.02} & 1.45 & 3.37 & \textbf{0.39} & & \\
         $t+3$ & 1.52 & 2.32 & \textbf{1.44} & 2.39 & 3.88 & \textbf{1.00} &   &  \\
         $t+4$  & 1.92 & 2.72 & \textbf{1.84} & 3.35 & 4.42 & \textbf{2.06} & & \\
         $t+5$  & 2.25 & 3.06 & \textbf{2.17} & 4.33 & 5.04 & \textbf{2.62} & & \\
         $t+6$  & 2.51 & 3.28 & \textbf{2.42} & 5.31 & 5.74 & \textbf{3.13} & & \\
          $t+7$  & 2.72 & 3.40 & \textbf{2.63} & 6.30 & 6.45 & \textbf{3.63} & & \\
    \end{tabular}}
    \caption{Performance for predicting $f$, $g_a^2$ and $g_e$ for the three comparing models, where \textit{Delay} is the Delay-SDE-net.}
    \label{tab:sim_results}
\end{table}

\subsection{A real-world case study}
\label{subsect:case_study}

In this section, the Delay-SDE-net is evaluated on a real-world case study, and the performance is compared to the SDE-net and the VAR model (as in Section\,\ref{subsubsection:sim_compare}). This case study is about short-term prediction of a real-world two-dimensional weather system. That is, based on a two dimensional time series of wind and temperature, we predict wind up to one week using the Delay-SDE-net. More information about the time series, and arguments of why we use this specific dynamical system to assess the Delay-SDE-net is provided in the following section.

\subsubsection{Data}
\label{subsubsect:data}

In this section we give a summary of the data used for our case-study, and explain why this two-dimensional data set is interesting to study.

It is well known that weather is a nonlinear dynamical system of chaos. As a consequence, even if weather variables have a certain degree of memory that can be used in weather prediction, reliable predictions are limited to about 10 days \citep{krishnamurthy2019predictability}. Recent research focus on enhancing long-term prediction of surface weather. A proper representation of stratospheric weather (the stratosphere is an area of the atmosphere about $15$\;km to $50$\;km above the surface) has potential to enhance long-term surface prediction, see for example \cite{blanc18}. In particular, the extreme stratospheric event called sudden stratospheric warming (SSW) is important in this respect \citep[][]{karpechko16,scaife22,hitchcock14}. The phenomenon of SSWs is a consequence of an abrupt disruption in the stratospheric Norther Hemisphere winter circulation. That is, in winter time the stratospheric North Pole is surrounded by high-speed western to easterly wind (U wind) that slows down and ultimately reverse about $6$ times per decade \citep{Pedatella2018}. See \cite{eggen2021} and references therein for a more thorough introduction.

In the following section, the Delay-SDE-net is assessed as a short-term prediction model for stratospheric (\textit{zonal}) U wind. This is considered as a highly relevant case because it challenges the Delay-SDE-net with a seasonal trend, nonlinearities, memory in time, a seasonal uncertainty (variance) pattern, as well as a having probability of extreme events such as SSWs. Based on standard definitions of SSWs \citep{butler2015defining}, the Delay-SDE-net is fit using daily zonal mean stratospheric temperature and U wind data at $10$\;hPa altitude and $60^{\circ}$N. Data specifications are given in Table\,\ref{tab:data set_specs}, and final daily zonal mean time series are prepared as presented in \cite{eggen2021}. Note that we have chosen to study the two-dimensional dynamical system of U wind and temperature based on identified cross correlations between the two meteorological variables in \cite{eggen2022}.

\begin{table}[hbt!]
    \centering
    {\tabcolsep3pt
    \begin{tabular}{c|c|c|c|c|c}
       Date & Grid & Pressure level & Time & Area & Unit  \\ \toprule
       \makecell{1 Jan. 1979 to\\31 Dec. 2018} &  $0.5^{\circ}$ & $10$\;hPa & \makecell{00:00, 06:00,\\12:00, 18:00} & \makecell{$60^{\circ}$N and\\$[-180^{\circ}\text{E},180^{\circ}\text{E})$}  & \makecell{Temperature: Kelvin,\\U wind: m/s}\\ \bottomrule
    \end{tabular}}
    \caption{Specifications of stratospheric temperature and U wind data. The specifications for the two values are similar, except units. Initial data sets are retrieved as ECMWF ERA-Interim reanalysis model products. This table is borrowed from \cite{eggen2022}.\label{tab:data set_specs}}
    \label{tab:data_sepcs}
\end{table}

\begin{figure}[hbt!]
    \centering
    \includegraphics[width=7cm]{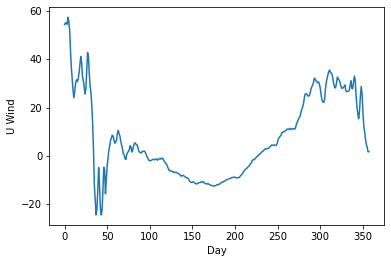}
    \caption{Daily zonal mean stratospheric U wind in 2018. There is a major SSW in February, which is visible by the sudden change in direction of U wind around day 45-51. The beginning of a new SSW can be observed in the end of the year.}
    \label{fig:stratos_warming}
\end{figure}

As indicated above, U wind has a seasonal pattern both in value and variance. In winter time U wind is usually having a positive high magnitude value, that changes to a negative low magnitude value in summer time. Further, the winter time stratosphere is more disturbed in winter time than in summer time, creating a higher aleatoric uncertainty in the stratospheric weather in winter time. An example year (2018) of daily zonal mean stratospheric U wind is shown in Figure\,\ref{fig:stratos_warming}, where the above points are clearly visible. Additionally, notice the major SSW in 2018, where the winter time stratosphere suddenly changes direction.

The results in the following section are based on a data set that is divided into a training set and a test set (30/10 split), where the model is trained on the years 1979-2008, and evaluated on the years 2009-2018.

\subsubsection{Results}
\label{subsubsect:results}

A successful prediction model should be able to provide high performance predictions of the U wind. Further, it should be able to identify SSWs, even though this might come from OOD data, and to predict the seasonal variance which is higher in winter time than in summer time. The Delay-SDE-net uses the deterministic trained drift net $f_m$ to predict the expected U wind, the aleatoric diffusion net $g_{a,m}$ to predict the seasonal variance, and the epistemic diffusion net $g_{e,m}$ to predict uncertainty coming from unusual events such as SSWs. When presenting the results, we refer to the observations $\{\bs{x}_{t_{k}}=[x_{1,{t_{k}}},x_{2,{t_{k}}}]:\quad t_{k}\in [1,365]\}$ as $x_1$ and $x_2$, where U wind, $x_1$, is the variable of interest, while temperature, $x_2$, is an explanatory variable.

The model is compared to the VAR and SDE-net models by predicting 1-7 days forward, and the results are given in Table \ref{tab:strat_results}. The drift is evaluated by the RMSE between the true $x_1$ and predicted U wind $\hat{x}_1$ (using Eq.\,\eqref{eq:training_values_of_drift_net}), while the diffusion is evaluated in terms of RMSE between the predicted variance $\hat{g}^2$ and the true residuals $e^2=(x_1-\hat{x}_1)^2$, to evaluate how well the models can predict their own uncertainty on the data. For the Delay-SDE-net, $\hat{g}^2=(g_{a,m}+g_{e,m})^2$. Note that we do not know the true $g_a$ and $g_e$ for this real-world data set, meaning we have to evaluate the models' ability to estimate its own total uncertainty, $\hat{g}^2=(\hat{g}_a+\hat{g}_e)^2$, where $\hat{g}_e=\sigma_e^{(N)}\hat{g}_e^{\text{prob}}$ for both the Delay-SDE-net and SDE-net. $\sigma_e^{(N)}$ was estimated according to Eq.  \ref{eq:min_sigma} for the results in Table \ref{tab:strat_results}, however in the plots of Figure\,\ref{fig:3_day}, \ref{fig:2018} and \ref{fig:2018_ci}, $\sigma_e^{(N)}$ was set to a higher value for illustration.

\begin{table}[hbt!]
    \centering
    \begin{tabular}{c|c|c|c|c|c|c}
         &  \multicolumn{3}{c|}{Value ($x_1|f$)} &  \multicolumn{3}{c}{Uncertainity ($g^2$)}\\ 
         &  \multicolumn{3}{c|}{(RMSE)} &  \multicolumn{3}{c}{(RMSE)}\\ \toprule
         &  VAR & SDE-net & Delay-SDE-net & VAR & SDE-net & Delay-SDE-net  \\ 
         $t+1$  & 2.31 & 2.25  & \textbf{1.48} & 15 & 14 & \textbf{6} \\
         $t+2$  & 3.37 & 3.70 & \textbf{2.92} & 34 & 36 & \textbf{24} \\
         $t+3$  & 4.40 & 4.64 & \textbf{4.10} & 58 & 63 & \textbf{44}  \\
         $t+4$ & 5.20 & 5.59 & \textbf{5.02} & 81 & 90 & \textbf{67}   \\
         $t+5$  & 5.84 & 6.77 & \textbf{5.70} & 104 & 130 & \textbf{90}  \\
         $t+6$  & 6.38 & 7.36 & \textbf{6.28} & 128 & 149 & \textbf{114} \\
         $t+7$  & 6.82 & 7.93 & \textbf{6.73} & 148 & 179 & \textbf{130}  \\  \bottomrule
    \end{tabular}
    \caption{1-7 day ahead prediction. \textit{Value} gives the RMSE between $x_1$ and $\hat{x_1}$, while \textit{Uncertainty} gives the RMSE between $\hat{g}^2=(\hat{g}_a+\hat{g}_e)^2$ and $e^2=(x_1-\hat{x}_1)^2$.}
    \label{tab:strat_results}
\end{table}

The Delay-SDE-net consistently provides the best prediction of both $f$ and $g$. The SDE-net has in general the lowest performance both in terms of $f$ and $g$, also lower than the VAR model, showing that the inclusion of time lags is more important than using nonlinear functions for this data set. The Delay-SDE-net works well by capturing nonlinear and complex functions by the use of neural networks, while  taking advantage of both the present and past states by behaving like an SDDE. However, we observe that the difference in performance for the Delay-SDE-net over the VAR model is higher in N-day ahead prediction for smaller $N$, while it decreases as we predict longer into the future. This is in correspondence with the knowledge that weather has limited length of inherent memory, as well as a large amount of natural chaos \citep{krishnamurthy2019predictability}, meaning all predictions will converge towards random guesses as $N$ increases.

To illustrate the Delay-SDE-net as a prediction model, Figure\,\ref{fig:3_day_a} shows the 3-day ahead prediction from the Delay-SDE-net on the test data, while Figure\,\ref{fig:3_day_b} shows the squared residuals together with the predicted uncertainties. It seems like the predicted U wind fits well with the true U wind, while $g_{a,m}^2$ captures uncertainty originating in seasonality and $g_{e,m}^2$ some of the residual peaks coming from unusual events such as SSW. However, as it is difficult to properly see the predictions on a time series over 10 years, we zoom in year 2012 as an example in Figure\,\ref{fig:2018}. This is a quite turbulent year, as it contains two minor SSWs (in the beginning and in the end of the year), as well as the beginning of a major SSW to the far end, who takes place in the beginning of 2013. We observe that the predicted U wind fits well, however making some mistakes around the high-variance winter seasons and especially around the SSWs. $g_{a,m}$ provides good predictions of the aleatoric uncertainty, both catching its seasonality as well as other effects given in the U wind and temperature values, and follows the pattern in the residuals in Figure\,\ref{fig:2012_b}. However, $g_{a,m}$ does not provide high enough predictions of the variances around the SSW before day 50 and at the end of the year, as it might not have seen many similar situations in the training data. We observe that $g_{e,m}$ predicts a high epistemic uncertainty at these situations, warning that something unusual might happen. Together, the two uncertainty predictions provide a good pattern, capturing much of the residuals in Figure\,\ref{fig:2012_c}.

\begin{figure}[hbt!]
\centering
\begin{subfigure}{.49\textwidth}
    \centering
    \includegraphics[width=.99\linewidth]{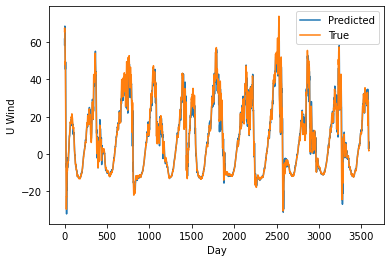}  
    \caption{}
    \label{fig:3_day_a}
\end{subfigure}
\begin{subfigure}{.49\textwidth}
    \centering
    \includegraphics[width=.99\linewidth]{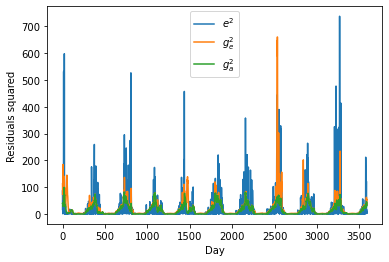}  
    \caption{}
    \label{fig:3_day_b}
\end{subfigure}
\caption{3-day ahead (a) true and predicted U Wind for 2009-2018 with the (b) predicted variances plotted together with the squared residuals. The green lines show $g_{a,m}^2$, while the orange line is $g_{e,m}^2$ added on top of $g_{a,m}^2$.}
\label{fig:3_day}
\end{figure}

By combining the prediction of expected U wind with predicted uncertainty, we can create prediction plots with confidence intervals. The 95\% confidence intervals are shown in Figure\,\ref{fig:2018_ci}, first for $g_{a,m}$ and $g_{e,m}$ separately, as well as the summed uncertainty in Figure\,\ref{fig:2018_ci_c}. As expected, $g_{a,m}$ provides a wide confidence in the volatile winter times, while going towards zero in summer time, where the U wind is much more stable. The confidence interval corresponding to $g_{e,m}$ is mainly narrow (as being trained as a probability, it is never actually zero), however it is wider around the unfamiliar SSWs. The combined plot in Figure\,\ref{fig:2018_ci_c} provides reasonable confidence intervals for the predicted U wind, and shows how the predictions from the neural networks $f_m$, $g_{a,m}$ and $g_{e,m}$ operate very well together.

\begin{figure}[hbt!]
\centering
\begin{subfigure}{.32\textwidth}
    \centering
    \includegraphics[width=.95\linewidth]{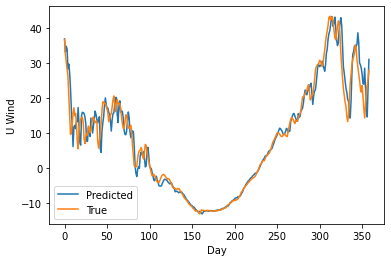}  
    \caption{}
    \label{fig:2012_a}
\end{subfigure}
\begin{subfigure}{.32\textwidth}
    \centering
    \includegraphics[width=.95\linewidth]{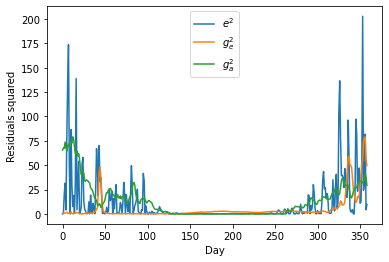}  
    \caption{}
    \label{fig:2012_b}
\end{subfigure}
\begin{subfigure}{.32\textwidth}
    \centering
    \includegraphics[width=.95\linewidth]{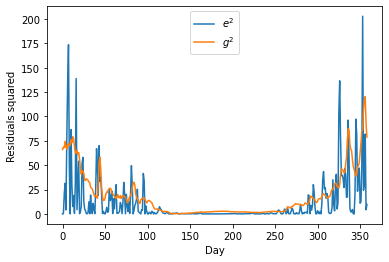}  
    \caption{}
    \label{fig:2012_c}
\end{subfigure}
\caption{3-day ahead (a) true and predicted U Wind for 2012 with the (b) squared residuals and predicted aleatoric and epistemic variances and (c) residuals and the total predicted variance.}
\label{fig:2018}
\end{figure}

\begin{figure}[hbt!]
\centering
\begin{subfigure}{.32\textwidth}
    \centering
    \includegraphics[width=.95\linewidth]{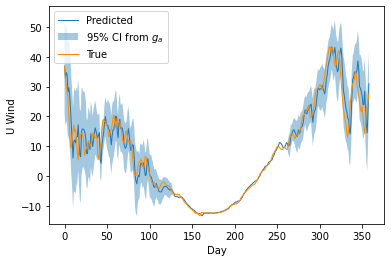}  
    \caption{}
    \label{fig:2018_ci_a}
\end{subfigure}
\begin{subfigure}{.32\textwidth}
    \centering
    \includegraphics[width=.95\linewidth]{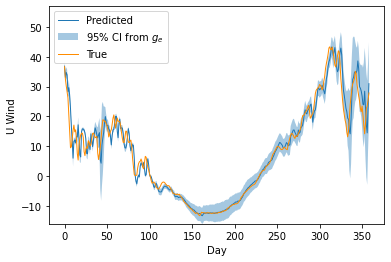}  
    \caption{}
    \label{fig:2018_ci_b}
\end{subfigure}
\begin{subfigure}{.32\textwidth}
    \centering
    \includegraphics[width=.95\linewidth]{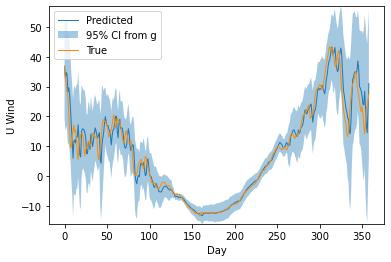}  
    \caption{}
    \label{fig:2018_ci_c}
\end{subfigure}
\caption{3-day ahead predicted U Wind for 2012 with 95\% confidence intervals calculated from (a) $g_{a_m}$, (b)  $g_{e,m}$ and (c) $g_m$.}
\label{fig:2018_ci}
\end{figure}


\section{Conclusions and further work}
\label{sect:conclusions}


This paper presents the Delay-SDE-net, a novel neural network model for predicting time series with memory, which also predicts aleatoric and epistemic uncertainty separately. The model is constructed to behave as a stochastic delay differential equation (SDDE), meaning it is based on a state-dependent drift term and a stochastic term providing predictions based on lagged data in the time series. The drift term is modelled by one neural network, and predicts deterministically the values of the time series. The stochastic term consists of two neural networks, where one models the aleatoric uncertainty in the time series and the other models the epistemic uncertainty.

A theoretical upper error bound is derived for the trained Delay-SDE-net compared to a real-world SDDE. The error bound consists of two parts, that is the upper error bound between the continuous-time Delay-SDE-net and a corresponding real-world SDDE, and the upper error bound between the continuous-time Delay-SDE-net and it discretized version. The former is inversely dependent on the number of neurons in the neural networks representing the model coefficients, and the latter is dependent on the resolution of the time discretization as usual. That is, we have proved that the model error can be adjusted to the better by increasing the number of neurons in training of the networks, and by using a finer time discretization scheme. The convergence rate of the discrete-time Delay-SDE-net towards the continuous-time Delay-SDE-net is studied numerically, and we find a rate of $\gamma=0.64$, which is better than the standard Euler-Maruyama convergence rate for It\^{o} diffusions.

The performance of the Delay-SDE-net is evaluated on data simulated from a real-world SDDE, as well as in a real-world case study involving a two-dimensional system of weather variables. That is, we compare predictions made by the Delay-SDE-net to corresponding predictions from the original SDE-net and a VAR model. The Delay-SDE-net has consistently the best performance, both in predicting the future values of the times series, as well as in estimating aleatoric and epistemic uncertainty.

As more fields and application areas are taking advantage of advanced statistical methods such as machine learning, it becomes more important to have models suitable for the requirements in each specific case. Safety-critical applications, for example, increase the need for providing prediction uncertainty, especially epistemic uncertainty, in order to make safe decisions also under unexpected operating conditions. Other areas, such as for example finance and weather modelling, often involve data with high inherent aleatoric uncertainty, which is important to estimate. Additionally, several application areas need close to immediate response of model predictions and estimated uncertainties to make decisions sufficiently fast. Examples include movement control of autonomous vehicles and stock trading. Therefore, the uncertainty estimates of the Delay-SDE-net are created by the single deterministic approach, meaning it can provide both prediction and uncertainty estimates immediately, without being delayed by repeated simulations.

The combined properties of the Delay-SDE-net, such as its ability to model time series with memory and give real-time prediction of uncertainty, currently makes it a unique model suitable for many fields that are searching for more advanced data-driven modelling techniques for time series. Additionally, restricting the system of networks to behave as an SDDE can potentially improve modelling properties such as interpretability, robustness, and stability on OOD data, as well as decrease the amount of required training data, similarly to other physics-informed neural networks \citep{karniadakis2021physics,kashinath2021physics,yang2019adversarial}.

Due to the novelty of the Delay-SDE-net, there are several aspects of training, modelling and numerical assessment that can be explored further. The authors are especially interested in further exploring the training methodology of the epistemic diffusion net. For example, one could use a regression procedure instead of classification, to avoid the tuning parameter $\sigma_e$. A more throughout investigation of the properties of the soft Brownian offset method in higher dimensions is also of interest, to conclude specific guidelines on appropriate tuning parameters for this approach on time series data. The convergence rate of the discreet-time Delay-SDE-net is assessed numerically in this work. An additional error assessment that should be performed is to study the error bound of the continuous-time Delay-SDE-net numerically, to verify that the error bound is inversely dependent on the number of neurons in the networks. Finally, verification of the model on other data sets, and within other fields, is also desirable. Regarding this last point, a work in progress uses the Delay-SDE-net on a complex data set of processed infrasound measurements to predict stratospheric wind. Preliminary predictions have reached cutting-edge results within this specific field. These results prove the Delay-SDE-net's ability to model complex dynamical systems.

\section*{Acknowledgements} 

A special thanks to Fred Espen Benth that has provided valuable advice and suggestions during this work. The authors are also grateful for the support from Riccardo De Bin, Quentin Brissaud, Arne Bang Huseby and Sven Peter N\"{a}sholm. M.D.E. was funded by a PhD grant from NORSAR. The authors have no conflicts of interest to declare.

\appendix
\section{Numerical convergence study}
\subsection{The continuous-time Delay-SDE-net reference solution}
\label{app:convergence_study_ref_solution}
The continuous time Delay-SDE-net is tested using $M_{\text{test}}$ samples. That is, let $s=1,\ldots ,M_{\text{test}}$, such that the data set of each sample path is given as 
\begin{align*}
    \mathcal{D}_s \coloneqq \{\bs{\eta}^{(s)}_{t_{-l}}\in\R^2:\quad t_{-l}\in[-15,0]\}\cup\{\bs{x}_{t_k}^{(s)}\in\R^2:\quad t_k\in(0,5]\},
\end{align*}
representing the real-world SDDE, $X(t)$. For each $\mathcal{D}_s$, predict $500$ steps forward to time $t_{N_{\text{ref}}}=5$, from time $t_0=0$, using the trained continuous time model $X^{0.01,(m)}(t)\simeq X^{(m)}(t)$ (where $X^{0.01,(m)}(t)=X^{\pi,(m)}(t)$ for $\Delta t_{\text{ref}}=0.01$). That is, for each $s$, use the trained two-layer Delay-SDE-net coefficients $f_m$ and $g_m=(g_{a,m}+g_{e,m})$ to compute
    \begin{align*}
        X^{(m)}(0.01)\simeq \hat{\bs{x}}_{t_1}^{(s)} &= \bs{\eta}_{t_0}^{(s)}+f_{m}(t_1,\bs{\eta}_{t_{0}}^{(s)},\bs{\eta}_{t_{-1}}^{(s)},\bs{\eta}_{t_{-2}}^{(s)}, \bs{\eta}_{t_{-3}}^{(s)})\Delta t_{\text{ref}} + g_{m}\Delta W_{\text{ref}}^{(1)}, \\
        X^{(m)}(0.02)\simeq \hat{\bs{x}}_{t_2}^{(s)} &= \hat{\bs{x}}_{t_1}^{(s)}+f_{m}(t_2,\hat{\bs{x}}_{t_1}^{(s)},\bs{\eta}_{t_{0}}^{(s)},\bs{\eta}_{t_{-1}}^{(s)},\bs{\eta}_{t_{-2}}^{(s)})\Delta t_{\text{ref}} + g_{m}\Delta W_{\text{ref}}^{(2)}, \\
        & \mathrel{\settowidth{\dimen0}{$=$}\hbox to \dimen0{\hss$\vdots$\hss}} \\
        X^{(m)}(5)\simeq \hat{\bs{x}}_{t_{N_{\text{ref}}}}^{(s)} &= \hat{\bs{x}}_{t_{499}}^{(s)}+f_{m}(t_{N_{\text{ref}}},\hat{\bs{x}}_{t_{499}}^{(s)},\ldots,\hat{\bs{x}}_{t_{496}}^{(s)})\Delta t_{\text{ref}} + g_{m}\Delta W_{\text{ref}}^{(N_{\text{ref}})},
    \end{align*}
where $g_{m}\coloneqq g_{m}(t_k,\bs{\eta}_{t_{0}}^{(s)},\bs{\eta}_{t_{-1}}^{(s)},\bs{\eta}_{t_{-2}}^{(s)},\bs{\eta}_{t_{-3}}^{(s)})$, $k=1,\ldots ,N_{\text{ref}}$, and $\Delta W_{\text{ref}}^{(k)} = W(t_{k+1}) - W(t_k)$, $\Delta W_{\text{ref}}^{(k)}\sim N(0,\Delta t_{\text{ref}})^d$. Now, the corresponding $L_2$-error is estimated as
\begin{align}
    \label{eq:estimated_cont_delay_sde_net_error}
    \Delta_{X^{(m)}}\coloneqq \left(\frac{1}{M_{\text{test}}}\sum_{s=1}^{M_{\text{test}}}\norm{\hat{\bs{x}}_{t_{N_{\text{ref}}}}^{(s)} - \bs{x}_{t_{N_{\text{ref}}}}^{(s)}}_2^2\right)^{1/2}.
\end{align}

\subsection{The discrete-time Delay-SDE-nets}
\label{app:convergence_study_discretized_models}
For each of these five models it is important to predict using the same Brownian path as for the generating SDDE and for the prediction made by the continuous Delay-SDE-net (the reference solution), such that all predictions are comparable. That is, for each $\kappa$ we use
\begin{align}
    \label{eq:convergence_analysis_bm}
    \Delta W_{\kappa}^{(\tilde{k})} = W(t_{\tilde{k}\kappa+\kappa}) - W(t_{\tilde{k}\kappa}) = \sum_{k=1}^{\kappa}\Delta W_{\text{ref}}^{(k)},
\end{align}
for $\tilde{k}=1,\ldots ,N_\kappa$, such that $\Delta W_{\kappa}^{(\tilde{k})}\sim N(0,\Delta t)^d$. Using this setup, we obtain $M_{\text{test}}$ predicted samples from each of the discretized Delay-SDE-net models $X^{\pi,(m)}(t)$, that is $X^{0.05,(m)}(t)$, $X^{0.1,(m)}(t)$, $X^{0.5,(m)}(t)$, $X^{1,(m)}(t)$ and $X^{5,(m)}(t)$. As an example, the following scheme illustrates how $\mathcal{D}_s$ is predicted from $X^{1,(m)}(t)$:\\
\begin{align*}
    X^{1,(m)}(1)= \hat{\bs{x}}_{t_{100}}^{(s)} &= \bs{\eta}_{t_0}^{(s)}+f_{m}(t_{100},\bs{\eta}_{t_{0}}^{(s)},\bs{\eta}_{t_{-100}}^{(s)},\bs{\eta}_{t_{-200}}^{(s)}, \bs{\eta}_{t_{-300}}^{(s)})\Delta t + g_{m}\Delta W_{100}^{(1)}, \\
    X^{1,(m)}(2)\simeq \hat{\bs{x}}_{t_{200}}^{(s)} &= \hat{\bs{x}}_{t_{100}}^{(s)}+f_{m}(t_{200},\hat{\bs{x}}_{t_{100}}^{(s)},\bs{\eta}_{t_{0}}^{(s)},\bs{\eta}_{t_{-100}}^{(s)},\bs{\eta}_{t_{-200}}^{(s)})\Delta t + g_{m}\Delta W_{100}^{(2)}, \\
    & \mathrel{\settowidth{\dimen0}{$=$}\hbox to \dimen0{\hss$\vdots$\hss}} \\
    X^{1,(m)}(5)\simeq \hat{\bs{x}}_{t_{N_{\text{ref}}}}^{(s)} &= \hat{\bs{x}}_{t_{400}}^{(s)}+f_{m}(t_{N_{\text{ref}}},\hat{\bs{x}}_{t_{400}}^{(s)},\ldots,\hat{\bs{x}}_{t_{100}}^{(s)})\Delta t + g_{m}\Delta W_{100}^{(N_{100})},
\end{align*}
where $g_{m}\coloneqq g_{m}(t_k,\bs{\eta}_{t_{0}}^{(s)},\bs{\eta}_{t_{-100}}^{(s)},\bs{\eta}_{t_{-200}}^{(s)}, \bs{\eta}_{t_{-300}}^{(s)})$, $k=100,200,\ldots ,N_{\text{ref}}$, $\Delta t = 1$, and $\Delta W_{100}^{(\tilde{k})}$, $\tilde{k}=1,\ldots ,5$ is defined according to Eq.\,\eqref{eq:convergence_analysis_bm}. Note that $\Delta_{X^{\pi,(m)}}$ is defined as in Eq.\,\eqref{eq:estimated_cont_delay_sde_net_error}.



\subsection{Model parameters} \label{a:num_param}

The final tuning parameters for the Delay-SDE-net used in the numerical convergence study are given in Table\,\ref{tab:num_tuning_converge}.

\begin{table}[hbt!]
    \centering
    \begin{tabular}{c|c}
    Parameter & Value \\ \toprule
        Learning rate $f_m$ & 0.01 \\ 
         Learning rate $g_{a,m}$ & 0.05 \\ 
         Iterations $f_m$ & 30 \\
         Iterations $g_{a,m}$ & 22 \\
         Momentum & 0.9 \\
         Weight decay &   0.00005 \\
         Optimizer & SDG \\\bottomrule
    \end{tabular}
    \caption{Tuning parameters for the Delay-SDE-net for the simulated data in \ref{subsubsect:numerical_convergence_study}.}
    \label{tab:num_tuning_converge}
\end{table}

\section{Comparison Study}
\subsection{OOD data} \label{a:ood_data}
The intervals where the test set $\mathcal{D}$ is replaced with values from the OOD data $\tilde{\mathcal{D}}$ in Section\,\ref{subsubsection:sim_compare} are given in Table\,\ref{tab:intervals}.

\begin{table}[hbt!]
    \centering
    \begin{tabular}{c|c|c|c|c|c|c|c|c|c|c}
        Year & 1 & 1 & 3 & 3 & 4 & 5 & 8 & 8 & 9 & 9 \\ \toprule
        Start $t_k$ &  23 & 313 & 79  & 344 & 275 & 67 & 1 & 190 & 48 & 323 \\
        End $t_k$ & 25 & 327 & 91 & 364 & 294 & 71 & 5 & 197 & 52 & 333\\ \bottomrule
    \end{tabular}
    \caption{The intervals where the test set $\mathcal{D}$ is replaced with values from $\tilde{\mathcal{D}}$ in the comparison study based on the simulated data}
    \label{tab:intervals}
\end{table}

\subsection{Model parameters} \label{a:comparison_param}

The final tuning parameters for the Delay-SDE-net and SDE-net for the comparison study on simulated data are given in Table\,\ref{tab:num_tuning} and \ref{tab:num_tuning_sde}, respectively. Some shared parameters for the two neural networks models are given in Table\,\ref{tab:both_tuning}.

\begin{table}[hbt!]
    \centering
    \begin{tabular}{c|c|c|c|c|c|c|}
         & lr $f_m$ & lr $g_{a,m}$ & lr $g_{e,m}$ & $\# f_m$ & $\# g_{a,m}$ & $\# g_{e,m}$   \\ \toprule
         $t+1$  & 0.01 & 0.001  & 0.005 & 500 & 500 & 20 \\
         $t+2$  & 0.005 & 0.0001 & & 500 & 500&  \\
         $t+3$  & 0.001 & 0.00005 &  & 500 & 500 &   \\
         $t+4$ & 0.001 & 0.000005 &  & 500 & 63 &    \\
         $t+5$  & 0.0005 & 0.000005 &  & 500 & 261 &   \\
         $t+6$  & 0.0005 & 0.000005 &  & 500& 279 &  \\
         $t+7$  & 0.0005 & 0.000005 &  & 500 & 303 & \\  \bottomrule
    \end{tabular}
    \caption{Tuning parameters for the Delay-SDE-net for the simulated data in \ref{subsubsection:sim_compare}. \textit{lr} is the learning rate and $\#$ is the number of training iterations.}
    \label{tab:num_tuning}
\end{table}

\begin{table}[hbt!]
    \centering
    \begin{tabular}{c|c|c|c}
         & lr $f_m$ & lr $g_{m}$ &  $\#$  \\ \toprule
         $t+1$  & 0.005 & 0.005  & 500  \\
         $t+2$  & 0.001 & 0.001 & 500   \\
         $t+3$  & 0.001 & 0.001 & 500  \\
         $t+4$ & 0.001 & 0.001 & 500     \\
         $t+5$  & 0.0005 & 0.0005 & 500  \\
         $t+6$  & 0.0005 & 0.0005 & 500   \\
         $t+7$  & 0.0005 & 0.0005 & 500  \\  \bottomrule
    \end{tabular}
    \caption{Tuning parameters for the SDE-net for the simulated data in \ref{subsubsection:sim_compare}. \textit{lr} is the learning rate and $\#$ is the number of training iterations.}
    \label{tab:num_tuning_sde}
\end{table}

\begin{table}[hbt!]
    \centering
    \begin{tabular}{c|c}
    Parameter & Value \\ \toprule
         Momentum & 0.9 \\
         Weight decay &   0.00005 \\
         Optimizer & SGD \\\bottomrule
    \end{tabular}
    \caption{Parameters for the Delay-SDE-net and SDE-net for the simulated data in \ref{subsubsection:sim_compare} and for the stratospheric data in \ref{subsubsect:results}. These were not tuned, but predetermined.}
    \label{tab:both_tuning}
\end{table}

\subsection{VAR model parameters} \label{a:var_param}

The VAR model parameters for the comparison study on simulated data are given in Table\,\ref{tab:var_model_parameters_sim}. The volatility function fit to VAR model residuals for the simulation study is different from the function used in \cite{eggen2022}. That is, let $y$ represent the expected value of squared VAR model residuals, and let $x$ represent the day of the year. Then, the fitted variance function is given by
\begin{align*}
    y = e^{a+bx},
\end{align*}
where $a$ and $b$ are constants. Optimal parameters $a$ and $b$ for each dimension are found using a least squares polynomial fit. See Table\,\ref{tab:variance_simulated_data_var_model}. 

\begin{table}[hbt!]
    \centering
    \begin{tabular}{c|c|c|c}
         $\hat{\phi}_1$ & $\hat{\phi}_2$ & $\hat{\phi}_3$ & $\hat{\phi}_4$  \\ \toprule
        $\begin{bmatrix} & X^{1,\pi} & X^{2,\pi} \\ X^{1,\pi} & 1.19 & -0.0062 \\ X^{2,\pi} & -0.37 & 1.06 \end{bmatrix}$ & 
        $\begin{bmatrix} 0.0015 & 0.0087 \\ 0.028 & -0.077 \end{bmatrix}$ &
        $\begin{bmatrix} -0.13 & -0.0015 \\ 0.31 & 0.025 \end{bmatrix}$ &
        $\begin{bmatrix} -0.17 & -0.0016  \\ 0.39 & -0.010 \end{bmatrix}$ \\  \bottomrule
    \end{tabular}
    \caption{VAR($4$) model parameters for the two dimensional dynamical system of simulated data $X^{\pi}=[X^{1,\pi},X^{2,\pi}]^T$ (see Section\,\ref{subsubsect:simulation_study_setup}). The positional markings for $\hat{\phi}_1$ hold for the other matrix-parameters as well.}
    \label{tab:var_model_parameters_sim}
\end{table}

\begin{table}[hbt!]
    \centering
    \begin{tabular}{c|c|c}
         & $X^{1,\pi}$ & $X^{2,\pi}$    \\ \toprule
         $a$  & $1.0628$ & $-4.3482$  \\
         $b$  & $-0.01694$ & $2.2593\times 10^{-5}$  \\  \bottomrule
    \end{tabular}
    \caption{Parameters for variance function fitted to expected value of squared VAR model residuals of simulated data.}
    \label{tab:variance_simulated_data_var_model}
\end{table}

\section{Real-world case study} \label{a:real_param}
\subsection{Model parameters}

The final tuning parameters for the Delay-SDE-net and SDE-net for the real-world case study are given in Table\,\ref{tab:real_tuning} and \ref{tab:real_tuning_sde}, respectively. Some shared parameters for the two neural networks models are given in Table\,\ref{tab:both_tuning}, which were the same as for the comparison study on simulated data.

\begin{table}[hbt!]
    \centering
    \begin{tabular}{c|c|c|c|c|c|c|}
         & lr $f_m$ & lr $g_{a,m}$ & lr $g_{e,m}$ & $\# f_m$ & $\# g_{a,m}$ & $\# g_{e,m}$   \\ \toprule
         $t+1$  & 0.05 & 0.01  & 0.01 & 500 & 486 & 20 \\
         $t+2$  & 0.01 & 0.01 & 0.01 & 466 & 114&  20 \\
         $t+3$  & 0.01 & 0.01 & 0.01  & 500 & 229 &  20  \\
         $t+4$ & 0.01 & 0.01 & 0.01 & 314 & 63 &   20  \\
         $t+5$  & 0.005 & 0.05 & 0.01 & 215 & 189 &  20  \\
         $t+6$  & 0.005 & 0.01 &  0.01& 312 & 500 &  20  \\
         $t+7$  & 0.005 & 0.01 & 0.01 & 213 & 500 &   20\\  \bottomrule
    \end{tabular}
    \caption{Tuning parameters for the Delay-SDE-net for the stratospheric data in \ref{subsubsect:results}. \textit{lr} is the learning rate and $\#$ is the number of training iterations.}
    \label{tab:real_tuning}
\end{table}

\begin{table}[hbt!]
    \centering
    \begin{tabular}{c|c|c|c}
         & lr $f_m$ & lr $g_{m}$ &  $\#$  \\ \toprule
         $t+1$  & 0.0005 & 0.0005  & 500  \\
         $t+2$  & 0.001 & 0.001 & 27   \\
         $t+3$  & 0.005 & 0.005 & 26  \\
         $t+4$ & 0.005 & 0.005 & 20     \\
         $t+5$  & 0.001 & 0.001 & 46  \\
         $t+6$  & 0.001 & 0.001 & 52   \\
         $t+7$  & 0.001 & 0.001 & 56  \\  \bottomrule
    \end{tabular}
    \caption{Tuning parameters for the SDE-net for the stratospheric data in \ref{subsubsect:results}. \textit{lr} is the learning rate and $\#$ is the number of training iterations.}
    \label{tab:real_tuning_sde}
\end{table}

\subsection{VAR model parameters}
A detailed introduction of the VAR model fit to the two dimensional system of U wind and temperature is given in \cite{eggen2022}. Please see the given reference for an explanation of given model parameters in this section. 

\begin{table}[hbt!]
    \centering
    \begin{tabular}{c|c|c|c|c|c|c|c|c|c|c}
         $c_0$ & $c_2$ & $c_4$ & $c_6$ & $c_8$ & $c_{10}$ & $c_{12}$ & $c_{14}$ & $c_{16}$ & $c_{18}$ & $c_{20}$  \\ \toprule
          $11.47$ & $0.41$ & $22.71$ & $0.20$ & $0.83$ & $0.67$ & $1.048$ & $1.13$ &$0.82$ & $0.60$ & $0.58$  \\  \midrule
          $c_1$ & $c_3$ & $c_5$ & $c_7$ & $c_9$ & $c_{11}$ & $c_{13}$ & $c_{15}$ & $c_{17}$ & $c_{19}$ & $c_{21}$  \\ \toprule
          $-0.00018$ & $-0.13$ & $-4.14$ & $-0.64$ & $-0.040$ & $-0.81$ & $0.64$ & $-0.18$ & $-0.0035$ & $0.49$ & $0.38$  \\  \bottomrule
    \end{tabular}
    \caption{Seasonality function parameters for U wind.}
    \label{tab:u_wind_season}
\end{table}

\begin{table}[hbt!]
    \centering
    \begin{tabular}{c|c|c|c|c|c|c|c|c|c|c}
         $c_0$ & $c_2$ & $c_4$ & $c_6$ & $c_8$ & $c_{10}$ & $c_{12}$ & $c_{14}$ & $c_{16}$ & $c_{18}$ & $c_{20}$  \\ \toprule
          $226.28$ & $-0.074$ & $-12.09$ & $0.20$ & $1.80$ & $0.11$ & $0.033$ & $-0.13$ & $-0.12$ & $-0.23$ & $0.093$ \\  \midrule
          $c_1$ & $c_3$ & $c_5$ & $c_7$ & $c_9$ & $c_{11}$ & $c_{13}$ & $c_{15}$ & $c_{17}$ & $c_{19}$ & $c_{21}$  \\ \toprule
           $-0.00010$ & $-0.088$ & $1.57$ & $-0.050$ & $2.80$ & $0.14$ & $1.40$ & $0.14$ & $0.19$ & $-0.040$ & $-0.099$ \\  \bottomrule
    \end{tabular}
    \caption{Seasonality function parameters for temperature.}
    \label{tab:temp_season}
\end{table}

\begin{table}[hbt!]
    \centering
    \begin{tabular}{c|c|c|c}
         $\hat{\phi}_1$ & $\hat{\phi}_2$ & $\hat{\phi}_3$ & $\hat{\phi}_4$  \\ \toprule
        $\begin{bmatrix} & \text{U wind} & \text{Temp.} \\ \text{U wind} & 1.72 & 0.0054 \\ \text{Temp.} & -0.10 & 1.53\end{bmatrix}$ & 
        $\begin{bmatrix} -1.03 & -0.024 \\ -0.082 & -0.73\end{bmatrix}$ &
        $\begin{bmatrix} 0.32 & -0.0002 \\ 0.11 & 0.27\end{bmatrix}$ &
        $\begin{bmatrix} -0.044 & 0.021  \\ 0.017 & -0.098 \end{bmatrix}$ \\  \bottomrule
    \end{tabular}
    \caption{VAR($4$) model parameters for the two dimensional dynamical system of U wind and temperature. The positional markings for $\hat{\phi}_1$ hold for the other matrix-parameters as well.}
    \label{tab:var_model_parameters}
\end{table}

\begin{table}[hbt!]
    \centering
    \begin{tabular}{c|c|c|c|c|c|c|c|c}
        \multicolumn{9}{c}{Winter/spring: $w_{0.30,2}^{(2)}$} \\
         $d_{0,2}^{(1)}$ & $d_{1,2}^{(1)}$ & $d_{2,2}^{(1)}$ & $d_{3,2}^{(1)}$ & $d_{4,2}^{(1)}$ & $-$ & $-$ & $-$ & $-$ \\ \toprule
            $-30.63$ & $-339.16$ &  $1987.79$ & $373.69$ & $-957.80$ & $-$ & $-$ & $-$ & $-$   \\  \midrule
          \multicolumn{9}{c}{Summer: $w_{0.05,2}^{(3)}$} \\
           $d_{0,2}^{(2)}$ & $d_{1,2}^{(2)}$ & $d_{2,2}^{(2)}$ & $d_{3,2}^{(2)}$ & $d_{4,2}^{(2)}$ & $d_{5,2}^{(2)}$ & $d_{6,2}^{(2)}$ & $-$ & $-$ \\ \toprule
            $55.88$ & $-198.28$ & $2460.26$ & $-472.34$ & $2420.60$ & $611.73$ & $-2390.93$ & $-$ & $-$   \\  \midrule
           \multicolumn{9}{c}{Autumn/winter: $w_{0.05,2}^{(4)}$} \\
           $d_{0,2}^{(3)}$ & $d_{1,2}^{(3)}$ & $d_{2,2}^{(3)}$ & $d_{3,2}^{(3)}$ & $d_{4,2}^{(3)}$ & $d_{5,2}^{(3)}$ & $d_{6,2}^{(3)}$ & $d_{7,2}^{(3)}$ & $d_{8,2}^{(3)}$ \\ \toprule
           $128.42$ & $104.79$ & $22.64$ & $37.76$ & $24.94$ & $-65.26$ & $-9.86$ & $-190.33$ & $-96.95$ \\ \bottomrule
    \end{tabular}
    \caption{Parameters for volatility function fitted to VAR model residuals of U wind.}
    \label{tab:uwind_vol}
\end{table}

\begin{table}[hbt!]
    \centering
    \begin{tabular}{c|c|c|c|c}
        \multicolumn{5}{c}{Winter/spring: $w_{0.44,1}^{(2)}$} \\
         $d_{0,1}^{(1)}$ & $d_{1,1}^{(1)}$ & $d_{2,1}^{(1)}$ & $d_{3,1}^{(1)}$ & $d_{4,1}^{(1)}$ \\ \toprule
         $-1.16$ & $-24.02$ & $101.47$ & $25.98$ & $-47.80$  \\  \midrule
          \multicolumn{5}{c}{Summer: $w_{2.0,1}^{(2)}$} \\
           $d_{0,1}^{(2)}$ & $d_{1,1}^{(2)}$ & $d_{2,1}^{(2)}$ & $d_{3,1}^{(2)}$ & $d_{4,1}^{(2)}$\\ \toprule
            $0.089$ & $0.10$ & $0.021$ & $0.033$ & $0.014$  \\  \midrule
           \multicolumn{5}{c}{Autumn/winter: $w_{0.44,1}^{(2)}$} \\
           $d_{0,1}^{(3)}$ & $d_{1,1}^{(3)}$ & $d_{2,1}^{(3)}$ & $d_{3,1}^{(3)}$ & $d_{4,1}^{(3)}$ \\ \toprule
           $16.58$ & $-344.37$ & $113.479$ & $119.90$ & $134.75$ \\ \bottomrule
    \end{tabular}
    \caption{Parameters for volatility function fitted to VAR model residuals of temperature.}
    \label{tab:temp_vol}
\end{table}

\bibliographystyle{rusnat}  
\bibliography{bibl}

\end{document}